%% file: paper.tex
\documentclass{article}

\usepackage{microtype}
\usepackage{graphicx}
\usepackage{subfigure}
\usepackage{booktabs} 

\usepackage{hyperref}

\usepackage[accepted]{icml2025}

\usepackage{amsmath}
\usepackage{amssymb}
\usepackage{mathtools}
\usepackage{amsthm}

\usepackage[capitalize,noabbrev]{cleveref}

\theoremstyle{plain}
\newtheorem{theorem}{Theorem}[section]

\newtheorem{lemma}[theorem]{Lemma}
\newtheorem{corollary}[theorem]{Corollary}
\theoremstyle{definition}

\newtheorem{assumption}[theorem]{Assumption}
\theoremstyle{remark}

\usepackage[textsize=tiny]{todonotes}

\usepackage[title]{appendix}
\usepackage{enumitem}
\usepackage{xparse}
\usepackage{ulem}
\usepackage{multirow}
\input{math_commands}

\newcommand{\ours}{{\fontfamily{qpl}\selectfont ZoAR}}
\newcommand{\rein}{{\fontfamily{qpl}\selectfont REINFORCE}}
\newcommand{\Var}{\operatorname{Var}}

\icmltitlerunning{Zeroth-Order Optimization is Secretly Single-Step Policy Optimization}

\begin{document}

\twocolumn[
\icmltitle{Zeroth-Order Optimization is Secretly Single-Step Policy Optimization}



\icmlsetsymbol{equal}{*}

\begin{icmlauthorlist}
\icmlauthor{Junbin Qiu}{hkust-gz}
\icmlauthor{Zhengpeng Xie}{hkust-gz}
\icmlauthor{Xiangda Yan}{xiaomi}
\icmlauthor{Yongjie Yang}{xiaomi}
\icmlauthor{Yao Shu}{hkust-gz}
\end{icmlauthorlist}

\icmlaffiliation{hkust-gz}{Hong Kong University of Science and Technology (Guangzhou)}
\icmlaffiliation{xiaomi}{Xiaomi Inc}

\icmlcorrespondingauthor{Yao Shu}{yaoshu@hkust-gz.edu.cn}

\icmlkeywords{Machine Learning, ICML}

\vskip 0.3in
]



\printAffiliationsAndNotice{}  

\renewcommand{\textit}[1]{{%
  \fontfamily{ppl}\itshape\selectfont #1%
}}
\renewcommand{\textbf}[1]{{%
  \fontfamily{ppl}\bfseries\selectfont #1%
}}

\input{workspace/main}

\bibliography{workspace/reference.bib}
\bibliographystyle{icml2025}

\newpage
\input{workspace/appendix}

\end{document}

%% file: math_commands.tex
\usepackage{amsmath,amsfonts,bm}

\def\1{\bm{1}}

\def\rvu{{\mathbf{i}}}

\def\rvu{{\mathbf{u}}}

\def\rvx{{\mathbf{x}}}

\def\rmI{{\mathbf{I}}}

\def\vzero{{\bm{0}}}

\def\vtheta{{\bm{\theta}}}

\def\va{{\bm{a}}}
\def\vb{{\bm{b}}}

\def\ve{{\bm{e}}}

\def\vg{{\bm{g}}}

\def\vm{{\bm{m}}}

\def\vu{{\bm{u}}}
\def\vv{{\bm{v}}}

\def\vx{{\bm{x}}}

\DeclareMathAlphabet{\mathsfit}{\encodingdefault}{\sfdefault}{m}{sl}
\SetMathAlphabet{\mathsfit}{bold}{\encodingdefault}{\sfdefault}{bx}{n}

\def\gF{{\mathcal{F}}}

\def\gH{{\mathcal{H}}}

\def\gN{{\mathcal{N}}}
\def\gO{{\mathcal{O}}}

\def\gR{{\mathcal{R}}}

\def\sB{{\mathbb{B}}}

\def\sR{{\mathbb{R}}}
\def\sS{{\mathbb{S}}}

\newcommand{\E}{\mathbb{E}}

\newcommand{\half}{\frac{1}{2}}
\newcommand{\circled}[1]{\raisebox{.5pt}{\textcircled{\raisebox{-.9pt} {#1}}}}
\newcommand{\aqty}[1]{\left\langle #1 \right\rangle}
\newcommand{\norm}[1]{\left\| #1 \right\|}

\NewDocumentCommand\qty{ t\big t\Big t\bigg t\Bigg g o d() d|| d\|\|}
{ 
	\IfBooleanTF{#1}{\let\ltag\bigl \let\rtag\bigr}{
		\IfBooleanTF{#2}{\let\ltag\Bigl \let\rtag\Bigr}{
			\IfBooleanTF{#3}{\let\ltag\biggl \let\rtag\biggr}{
				\IfBooleanTF{#4}
				{\let\ltag\Biggl \let\rtag\Biggr}
				{\let\ltag\left \let\rtag\right}
			}
		}
	}
	\IfNoValueTF{#5}{
		\IfNoValueTF{#6}{
			\IfNoValueTF{#7}{
				\IfNoValueTF{#8}
				{()}
				{\ltag\lvert{#8}\rtag\rvert}
			}
			{\ltag(#7\rtag) \IfNoValueTF{#8}{}{|#8|}}
		}
		{\ltag[#6\rtag] \IfNoValueTF{#7}{}{(#7)} \IfNoValueTF{#8}{}{|#8|}}
	}
	{\ltag\lbrace#5\rtag\rbrace  \IfNoValueTF{#6}{}{[#6]} \IfNoValueTF{#7}{}{(#7)} \IfNoValueTF{#8}{}{|#8|}}
}

%% file: workspace/main.tex
\begin{abstract}
Zeroth-Order Optimization (ZOO) provides powerful tools for optimizing functions where explicit gradients are unavailable or expensive to compute. However, the underlying mechanisms of popular ZOO methods, particularly those employing randomized finite differences, and their connection to other optimization paradigms like Reinforcement Learning (RL) are not fully elucidated. 
This paper establishes a fundamental and previously unrecognized connection: ZOO with finite differences is equivalent to a specific instance of single-step Policy Optimization (PO). We formally unveil that the implicitly smoothed objective function optimized by common ZOO algorithms is identical to a single-step PO objective. Furthermore, we show that widely used ZOO gradient estimators, are mathematically equivalent to the \rein{} gradient estimator with a specific baseline function, revealing the variance-reducing mechanism in ZOO from a PO perspective.Built on this unified framework, we propose \ours{} (\textit{Zeroth-Order Optimization with Averaged Baseline and Query Reuse}), a novel ZOO algorithm incorporating PO-inspired variance reduction techniques: an averaged baseline from recent evaluations and query reuse analogous to experience replay. 
Our theoretical analysis further substantiates these techniques reduce variance and enhance convergence.
Extensive empirical studies 
validate our 
theory
and demonstrate that \ours{} significantly outperforms other methods in terms of convergence speed and final performance.
Overall, our work provides a new theoretical lens for understanding ZOO and offers practical algorithmic improvements derived from its connection to PO.

\end{abstract}


\section{Introduction}
\label{sec:intro}

Zeroth-Order Optimization (ZOO) addresses the task of optimizing objectives $F(\vtheta) = \mathbb{E}_{\xi} [f(\vtheta; \xi)]$ using only function evaluations, bypassing the need for explicit gradients \citep{Nesterov2017, GhadimiL13a}. This paradigm is essential in numerous domains where gradients are intractable, computationally prohibitive, or simply unavailable, such as hyperparameter optimization \citep{gu2021optimizing}, derivative-free reinforcement learning \citep{SalimansHCS17}, communication-efficient federated learning \citep{fzoos}, black-box adversarial attacks \citep{zord, radazo}, prompt optimization \citep{zopo, unlocking}, and memory-efficient finetuning for large language models (LLMs) \citep{mezo, revisiting}. A dominant strategy within ZOO involves estimating gradients via randomized finite differences, which implicitly optimize a smoothed surrogate $F_{\mu}(\vtheta)$ of the original objective $F(\vtheta)$ \citep{Nesterov2017, radazo}. A thorough discussion on the most related works of ZOO is in Appx.~\ref{sec:related}. While foundational, these methods often suffer from high variance in their gradient estimates, potentially impeding convergence speed and solution quality. Furthermore, a deep theoretical understanding connecting these ZOO techniques to established principles in related fields like Reinforcement Learning (RL) remains underdeveloped. In parallel, Policy Optimization (PO) forms the bedrock of modern RL, seeking policy parameters $\vtheta$ to maximize expected cumulative rewards $J(\vtheta)$ \citep{sutton1999policy, sutton2018reinforcement}. Policy Gradient (PG) algorithms like REINFORCE \citep{williams1992simple} estimate $\nabla J(\vtheta)$ from trajectory rollouts. A crucial technique for stabilizing PG methods is baseline subtraction, which provably reduces gradient estimate variance and thereby accelerates learning \citep{sutton2018reinforcement}.

As the \textbf{first} primary contribution, this paper establishes a fundamental and \textit{previously unrecognized} connection: \textit{smoothed Zeroth-Order Optimization (ZOO) with finite differences is formally equivalent to a specific instance of single-step Policy Optimization (PO)}. We bridge these two fields, providing theoretical clarification for ZOO mechanisms (Sec.~\ref{sec:po-framework}): First, we {formally unveil} that the smoothed objective $F_{\mu}(\vtheta)$ implicitly targeted by common ZOO methods is \textit{identical} to a single-step PO objective $J(\vtheta)$ under a specific reward definition (Thm.~\ref{thm:obj-equiv}). Second, we {prove for the first time} that the standard Gaussian-smoothed ZOO gradient estimator is mathematically \textit{equivalent} to the single-step REINFORCE estimator using the function value $f(\vtheta; \xi)$ as a baseline (Thm.~\ref{thm:gs-equiv}). This {novel interpretation} recasts the standard ZOO baseline subtraction not merely as a finite-difference artifact, but as a principled variance reduction technique rooted in PO theory, {revealing the variance-reducing mechanism in ZOO from a PO lens}. Third, we {further extend this foundational equivalence} using importance sampling (Thm.~\ref{thm:zoo-equiv-is}), clarifying how ZOO estimators with alternative sampling distributions relate to weighted REINFORCE and optimize distinct smoothed objectives.

Building upon this {newly established unified PO framework}, our \textbf{second} primary contribution is \ours{} (\textit{Zeroth-Order Optimization with Averaged Baseline and Query Reuse}) proposed in Sec.~\ref{sec:zoar}. \ours{} {is the first to integrate} two PO-inspired variance reduction techniques directly into conventional ZOO (see Sec.~\ref{sec:algo-design}): 
\textit{(a) Averaged Baseline:} Instead of the high-variance single-point estimate $f(\vtheta; \xi)$, \ours{} introduces an averaged baseline from recent function evaluations in a history buffer. This novel ZOO adaptation of the value function estimation in PO provides a more stable Monte Carlo estimate of the smoothed objective $F_{\mu}(\vtheta)$.
\textit{(b) Query Reuse:} \ours{} computes gradient estimates using all samples in the history buffer (analogous to the experience replay in PO), effectively increasing the batch size for gradient estimation without new queries per iteration, thus enhancing sample efficiency and mitigating variance.
We further provide rigorous theoretical analysis in Appx.~\ref{appx:theory} to support the variance reduction effect of these two {newly introduced} PO-inspired techniques from the lens of ZOO theory and show the potentially improved convergence of \ours{} when variance dominates.

Our \textbf{third} contribution lies in comprehensive empirical validation (Sec.~\ref{sec:main-exp}). 
We benchmark \ours{} against other ZOO baselines, across standard synthetic functions, a black-box adversarial attack task, and memory-efficient finetuning of LLMs. 
The results consistently show that \ours{} achieves significant improvements in convergence rate and final performance, validating the practical efficacy of leveraging {these newly connected} PO techniques for ZOO. Notably, substantial gains are observed even with {our novel averaged baseline alone}, highlighting its distinct effectiveness.


\section{Preliminaries}\label{sec:prelimilaries}
This section introduces the necessary background on Zeroth-Order Optimization (ZOO) and Policy Optimization (PO) in Reinforcement Learning (RL), establishing the notation and core concepts used throughout the paper.

\textbf{Problem Setup.} We focus on the problem of minimizing a potentially non-convex objective function $F(\vtheta)$ defined as an expectation over a random variable $\xi$:
\begin{equation}
\min_{\vtheta \in \sR^d} F(\vtheta) \triangleq \mathbb{E}_{\xi}\left[f(\vtheta; \xi)\right] \label{eq:zo-obj} \ .
\end{equation}
Here, $\vtheta \in \sR^d$ represents the $d$-dimensional parameter vector we aim to optimize, $f(\vtheta; \xi)$ is a scalar-valued loss function whose evaluation depends on both the parameters $\vtheta$ and a random variable $\xi$. The defining characteristic of the Zeroth-Order (ZO) setting is the constraint that we can only access stochastic evaluations of the function value, $f(\vtheta; \xi)$, through a black-box oracle. Importantly, direct access to the gradient $\nabla_{\vtheta} f(\vtheta; \xi)$ is assumed to be unavailable or computationally prohibitive. Throughout this paper, we use $\nabla$ to denote the gradient with respect to the parameters $\vtheta$, i.e., $\nabla \equiv \nabla_{\vtheta}$.

\textbf{Zeroth-Order Optimization.}
To optimize \eqref{eq:zo-obj} without explicit gradients, ZOO algorithms employ gradient estimators constructed solely from function evaluations. A prevalent technique is randomized finite differences. A common form of such an estimator, averaged over $K$ directions is:
\begin{equation}\label{eq:fd}
\hat{\nabla} F(\vtheta) \triangleq \frac{1}{K}\sum_{k=1}^K\frac{f(\vtheta + \mu \rvu_k; \xi) - f(\vtheta; \xi)}{\mu} \rvu_k \ .
\end{equation}
where $\{\rvu_k\}_{k=1}^K$ are i.i.d. random direction vectors, $\mu > 0$ is a small smoothing radius parameter, and $K \ge 1$ dictates the number of function queries used per gradient estimate (beyond the baseline evaluation $f(\vtheta; \xi)$). Standard choices for the distribution of $\rvu_k$ include:
\begin{itemize}
[topsep=0pt,leftmargin=10mm,itemsep=0pt]
    \item[(I)] The standard multivariate Gaussian distribution $\rvu_k \sim \mathcal{N}(\vzero, \rmI_d)$ \citep{Nesterov2017}.
    \item[(II)] The uniform distribution over the unit sphere $\rvu_k \sim \mathrm{Unif}(\sS^{d-1})$ \citep{bsg}.
    \item[(III)] The uniform distribution over the standard basis vectors $\rvu_k \sim \mathrm{Unif}(\{\ve_1, \dots, \ve_d\})$ \citep{coordinate}.
\end{itemize}

It is well-established that \eqref{eq:fd} is an unbiased gradient estimation of a smoothed approximation $F_{\mu}$ (defined as below) for the original objective $F(\vtheta)$ \citep{Nesterov2017, radazo}. This means that ZOO with estimator \eqref{eq:fd} is in fact implicitly optimizing the smoothed objective  $F_{\mu}$.
\begin{equation}\label{eq:fu}
F_{\mu}(\vtheta) \triangleq \mathbb{E}_{\rvu}\left[F(\vtheta + \mu \rvu)\right]  = \mathbb{E}_{\rvu}\left[\mathbb{E}_{\xi}\left[f(\vtheta + \mu \rvu; \xi)\right]\right] \ .
\end{equation}

\textbf{Policy Optimization and \rein{}.} 
In policy optimization, the objective is typically to find the parameters $\vtheta$ of a stochastic policy $\pi_{\vtheta}(a|s)$ that maximize the expected cumulative reward. Let us consider the standard episodic setting. The objective function, $J(\vtheta)$, is the expected total discounted reward obtained by executing the policy $\pi_{\vtheta}$ starting from an initial state distribution $\rho_0(s_0)$:
\begin{equation}
\begin{aligned}
    &J(\vtheta) \triangleq \mathbb{E}_{\tau \sim p_{\vtheta}(\tau)}\left[\sum_{t=0}^{T-1} \gamma^{t} R(S_t, A_t)\right] \\
    =& \mathbb{E}_{S_0 \sim \rho_0, A_t \sim \pi_{\vtheta}(\cdot|S_t), S_{t+1} \sim P(\cdot|S_t, A_t)}\left[\sum_{t=0}^{T-1} \gamma^{t} R(S_t, A_t)\right]  \ .
\end{aligned}
\label{eq:pg-obj}
\end{equation}
Here, 
$\tau = (S_0, A_0, R_0, \dots, S_{T-1}, A_{T-1}, R_{T-1}, S_T)$ 
represents a trajectory (or episode) of states $S_t$, actions $A_t$, and rewards $R_t = R(S_t, A_t)$. The trajectory distribution $p_{\vtheta}(\tau)$ is induced by the policy $\pi_{\vtheta}$ and the transition dynamics $P(S_{t+1}|S_t, A_t)$ of environment. $\gamma \in [0, 1]$ is the discount factor, and $T$ is the episode horizon (which can be finite or infinite). Note that while policy optimization typically involves maximization, we can frame it as minimization by considering the negative reward (cost), i.e., minimizing $-J(\vtheta)$, to align with the optimization setup in \eqref{eq:zo-obj}.

Policy Gradient methods are a class of algorithms designed to optimize $J(\vtheta)$ by estimating its gradient $\nabla J(\vtheta)$ and performing gradient ascent (or descent on $-J(\vtheta)$). The Policy Gradient Theorem \citep{sutton1999policy} provides the analytical form of this gradient and a widely used policy gradient is derived from the \rein{} (w/ baseline) algorithm \citep{williams1992simple}:
\begin{equation}\label{eq:reinforce-baseline}
\nabla J(\vtheta) = \mathbb{E}_{\tau \sim p_{\vtheta}(\tau)}\left[ \sum_{t=0}^{T-1} \nabla \ln \pi_{\vtheta}(A_{t} | S_{t}) \left( G_t - b(S_t) \right) \right]
\end{equation}
where $G_t = \sum_{t'=t}^{T-1}\gamma^{t'-t} R(S_{t'}, A_{t'})$ represents the discounted return-to-go from time step $t$ and the state-dependent baseline $b(S_t)$ is applied for variance reduction. 

\section{A Policy Optimization Framework for Zeroth-Order Optimization}\label{sec:po-framework}
Building on the preliminaries in Sec.~\ref{sec:prelimilaries}, this section formally establishes the connection between Zeroth-Order Optimization (ZOO) and Policy Optimization (PO). We demonstrate that the ZOO problem
can be precisely framed as a single-step PO problem (Sec.~\ref{sec:equiv-obj}). Furthermore, we show that common ZOO gradient estimators are equivalent to specific instances of the \rein{} algorithm with a baseline (Sec.~\ref{sec:gs-rein} \& Sec.~\ref{sec:gen-rein}).

\subsection{Equivalence of Objectives in ZOO and PO}\label{sec:equiv-obj}
We begin by demonstrating the equivalence between the objective function implicitly optimized by many ZOO methods, i.e., $F_{\mu}(\vtheta)$ in \eqref{eq:fu}, and a specific instance of the PO objective. Formally, consider the standard PO objective from \eqref{eq:pg-obj} in a simplified, single-step episodic setting (i.e., $T=1$, $\gamma=0$). In this scenario, the agent takes a single action $\rvx$ sampled from a policy $\pi_{\vtheta}(\rvx)$, and receives a reward based on this action. To align with the minimization problem \eqref{eq:zo-obj}, we define the reward as the negative function value, $R_0 = -F(\rvx)$. The PO objective is then to minimize the expected negative reward:
\begin{equation}
J(\vtheta) \triangleq \mathbb{E}_{\rvx \sim \pi_{\vtheta}(\rvx)}\left[F(\rvx)\right] = \mathbb{E}_{\rvx \sim \pi_{\vtheta}(\rvx)}\left[\mathbb{E}_{\xi}\left[f(\vtheta; \xi)\right]\right] \ . \label{eq:pg-zo-obj}
\end{equation}

The connection between the ZOO smoothed objective $F_{\mu}(\vtheta)$ defined in \eqref{eq:fu} and this single-step PO objective $J(\vtheta)$ defined in \eqref{eq:pg-zo-obj} is formalized below (proof in Appx.~\ref{appx:proof-of-obj-equiv}).
\begin{theorem}[Objective Equivalence]\label{thm:obj-equiv}
\textit{\fontfamily{ppl}\selectfont
Let the policy $\pi_{\vtheta}(\rvx)$ be defined via the reparameterization $\rvx = \vtheta + \mu \rvu$, where $\rvu$ is a random vector drawn from a distribution $p(\rvu)$ independent of $\vtheta$. Then, the single-step PO objective $J(\vtheta)$ defined in \eqref{eq:pg-zo-obj} is identical to the ZOO smoothed objective $F_{\mu}(\vtheta)$ defined in \eqref{eq:fu} using the same distribution $p(\rvu)$, i.e., 
\begin{equation*}
J(\vtheta) = F_{\mu}(\vtheta) \ .
\end{equation*}
}
\end{theorem}
\textbf{Remark.} Thm. \ref{thm:obj-equiv} establishes that optimizing the smoothed function $F_{\mu}(\vtheta)$, a standard practice in ZOO theory, is equivalent to optimizing a single-step RL objective $J(\vtheta)$ where the policy samples perturbations around the current parameters $\vtheta$. This equivalence allows us to leverage concepts and algorithms from PO to understand and potentially improve ZOO methods (see Sec.~\ref{sec:zoar}). The choice of the smoothing distribution $p(\rvu)$ in ZOO corresponds to the choice of the exploration strategy (policy structure) in this PO context. To the best of our knowledge, this is the first to explicitly interpret the ZOO smoothed objective through this specific PO lens.


\subsection{Gaussian Smoothing as Single-Step \rein{} w/ Baseline}
\label{sec:gs-rein}

We now demonstrate that the widely used Gaussian-smoothed ZOO gradient estimator is equivalent to a specific instance of the \rein{} w/ baseline algorithm. Let the smoothing distribution be the standard multivariate Gaussian, $p(\rvu) = \mathcal{N}(\vzero, \rmI_d)$. The corresponding policy $\pi_{\vtheta}(\rvx)$ samples $\rvx = \vtheta + \mu \rvu$, which means $\rvx \sim \mathcal{N}(\vtheta, \mu^2\rmI_d)$. To minimize $F_{\mu}(\vtheta) = J(\vtheta)$, We apply the \rein{} w/ baseline algorithm using the policy gradient theorem \eqref{eq:reinforce-baseline}. For our single-step case ($T=1$), the policy gradient gives:
\begin{equation}\label{eq:single-pg}
\nabla J(\vtheta) = \mathbb{E}_{\rvx \sim \pi_{\vtheta}(\rvx)}\left[ \nabla \ln \pi_{\vtheta}(\rvx) \left( \mathbb{E}_{\xi}[f(\rvx; \xi)] - b \right) \right] \ ,
\end{equation}
where $b$ is a baseline that is independent of the specific sample $\rvx$. Particularly, for the Gaussian policy $\pi_{\vtheta}(\rvx) = \mathcal{N}(\vtheta, \mu^2\rmI_d)$, we have $\nabla \ln \pi_{\vtheta}(\rvx) = \frac{\rvx - \vtheta}{\mu^2}$. Substituting this into \eqref{eq:single-pg} gives:
\begin{equation}
\nabla J(\vtheta) = \mathbb{E}_{\rvx \sim \pi_{\vtheta}(\rvx)}\left[ \frac{\rvx - \vtheta}{\mu^2} \left( \mathbb{E}_{\xi}[f(\rvx; \xi)] - b \right) \right] \ .
\end{equation}
In practice, the expectations are approximated using Monte Carlo sampling. Let $ b = \E_{\xi}\left[b(\xi)\right]$, we sample $\rvx_k$ from $\pi_{\vtheta}(\rvx) $ to estimate the outer expectation and $\xi$ to estimate the inner expectation. A common stochastic gradient estimator based on $K$ samples is then:
\begin{equation}
\hat{\nabla}_{\text{GS}} J(\vtheta) \triangleq \frac{1}{K}\sum_{k=1}^{K} \frac{\rvx_k - \vtheta}{\mu^2} \left(f(\rvx_k; \xi) - b(\xi)\right) \ . \label{eq:reinforce-grad-refined}
\end{equation}

The connection between the standard Gaussian-smoothed ZOO gradient estimator from \eqref{eq:fd} and the \rein{} gradient estimator \eqref{eq:reinforce-grad-refined} is formalized below (proof in Appx.~\ref{appx:proof-of-gs-equiv}).
\begin{theorem}[Gradient Estimator Equivalence for Gaussian]\label{thm:gs-equiv}
\textit{\fontfamily{ppl}\selectfont
Let $\pi_{\vtheta}(\rvx) = \mathcal{N}(\vtheta, \mu^2\rmI_d)$ and $b(\xi) = f(\vtheta; \xi)$ in \eqref{eq:reinforce-grad-refined}. Then, the \rein{} gradient estimator \eqref{eq:reinforce-grad-refined} is identical to the Gaussian-smoothed ZOO gradient estimator \eqref{eq:fd}, i.e., 
\begin{equation*}
 \hat{\nabla}_{\normalfont{\text{GS}}} J(\vtheta)  = \hat{\nabla} F(\vtheta)\ .
\end{equation*}
}
\end{theorem}
\textbf{Remark.} Thm. \ref{thm:gs-equiv} provides the first explicit interpretation of the common ZOO gradient estimator \eqref{eq:fd} from a novel PO lens. Specifically, it reveals that Gaussian-smoothed ZOO estimator can be interpreted as \rein{} gradient estimator with gaussian policy. Moreover, 
it unveils that the subtraction of $f(\vtheta; \xi)$ in conventional ZOO is not merely a result from the first-order Taylor polynomial but corresponds precisely to using a baseline in the \rein{} algorithm. This baseline is known to reduce the variance of the gradient estimate without introducing bias \citep{sutton2018reinforcement}. This perspective not only aligns with but also provides a theoretical support for observations in works like \citep{SalimansHCS17} where similar estimators were used in the context of evolution strategies, highlighting the variance reduction benefit without explicitly linking it to the \rein{} w/ baseline mechanism.


\subsection{Generalization Through Importance Sampling}\label{sec:gen-rein}
The previous section only established the equivalence for Gaussian smoothing, whereas ZOO methods can also apply other sampling distributions for $\rvu_k$, like the uniform distribution over the unit sphere or coordinate directions mentioned in Sec.~\ref{sec:prelimilaries}. We hence generalize our PO perspective to encompass these cases using importance sampling (IS) in this section.

Suppose we still consider the objective $J(\vtheta)$ with the Gaussian policy $\pi_{\vtheta}(\rvx) = \mathcal{N}(\vtheta, \mu^2\rmI_d)$, but we want to estimate its gradient using samples drawn from a different proposal distribution $p(\rvx)$. The policy gradient using importance sampling becomes:
\begin{equation}
\nabla J(\vtheta) = \mathbb{E}_{\rvx \sim p(\rvx)}\left[ \frac{\pi_{\vtheta}(\rvx)}{p(\rvx)} \nabla \ln \pi_{\vtheta}(\rvx) \left( \mathbb{E}_{\xi}[f(\rvx; \xi)] - b \right) \right] \ .
\end{equation}
Similar to \eqref{eq:reinforce-grad-refined}, by substituting $\nabla \ln \pi_{\vtheta}(\rvx) = \frac{\rvx - \vtheta}{\mu^2}, b=\E_{\xi}\left[b(\xi)\right]$ and using Monte Carlo approximation with samples $\rvx_k \sim p(\rvx)$, we get the stochastic gradient estimator:
\begin{equation}
\hat{\nabla}_{\text{IS}} J(\vtheta) \triangleq \frac{1}{K}\sum_{k=1}^{K} \frac{\pi_{\vtheta}(\rvx_k)}{p(\rvx_k)} \frac{\rvx_k - \vtheta}{\mu^2} \left(f(\rvx_k; \xi) - b(\xi)\right) \ . \label{eq:reinforce-grad-is}
\end{equation}

The connection between the ZOO gradient estimator under various sampling distributions from \eqref{eq:fd} and the IS-based REINFORCE gradient estimator \eqref{eq:reinforce-grad-is} is formalized below (proof in Appx.~\ref{appx:proof-of-zoo-equiv-is}).
\begin{theorem}[Extended Gradient Estimator Equivalence]\label{thm:zoo-equiv-is}
\textit{\fontfamily{ppl}\selectfont
Let $\pi_{\vtheta}(\rvx) = \mathcal{N}(\vtheta, \mu^2\rmI_d)$, $p(\rvx) = p(\vtheta + \mu\rvu)$, and $b(\xi) = f(\vtheta; \xi)$ in \eqref{eq:reinforce-grad-is}. IS-based \rein{} gradient estimator \eqref{eq:reinforce-grad-is} is identical to a scaled ZOO gradient estimator \eqref{eq:fd} for the three different distributions of $\rvu_k$ in Sec.~\ref{sec:prelimilaries}:
\begin{equation*}
\hat{\nabla}_{\normalfont{\text{IS}}} J(\vtheta) = \gamma \hat{\nabla} F(\vtheta) \ .
\end{equation*}
Particularly, let $\Gamma(\cdot)$ be the Gamma function, \normalfont{(I)} if  $\rvu_k \sim \mathcal{N}(\vzero, \rmI_d)$, $\gamma = 1$; \normalfont{(II)} if $\rvu_k \sim \mathrm{Unif}(\sS^{d-1})$, $\gamma = \frac{2^{1-d/2}\exp(-1/2)}{\mu \Gamma(d/2)}$; \normalfont{(III)} if $\rvu_k \sim \mathrm{Unif}(\{\ve_1, \dots, \ve_d\})$, $\gamma = \frac{d\exp(-1/2)}{(2\pi \mu^2)^{d/2}}$.
}
\end{theorem}
\textbf{Remark.} Thm. \ref{thm:zoo-equiv-is} reveals that ZOO estimators employing non-Gaussian sampling distributions for $\rvu_k$ (e.g., uniform on sphere or coordinate-wise) can also be interpreted as \rein{} gradient estimators through the lens of importance sampling. Specifically, the ZOO gradient $\hat{\nabla} F(\vtheta)$ (unbiased for its own smoothed objective $F_\mu(\vtheta)$ with the non-Gaussian $p(\rvu)$) remains equivalent to an IS-based REINFORCE estimator for $J(\vtheta)$ with the Gaussian policy scaled by $\gamma$. This scaling factor $\gamma$ arises from the implicit importance weights between the Gaussian policy for $J(\vtheta)$ and the ZOO proposal distribution $p(\rvu)$. This perspective unifies diverse ZOO sampling strategies under the \rein{} lens, provides a principled reason for the learning rate adjustments in Cor.~\ref{cor:equiv-conv}, and further solidifies the fundamental equivalence between the convergence of ZOO and single-step PO.
\begin{corollary}[Convergence Equivalence]\label{cor:equiv-conv}
\textit{\fontfamily{ppl}\selectfont
Under the same condition in  Thm.~\ref{thm:zoo-equiv-is}, let baseline $b(\xi)$ and update rule (e.g. gradient descent algorithm and Adam~\citep{adam} algorithm) be the same for ZOO and \rein{}, they achieve identical convergence when
\begin{equation}
  \eta_{\text{R}} = \eta_{\text{Z}} / \gamma \ . \nonumber
\end{equation}
Here, $\gamma$ is from Thm.~\ref{thm:zoo-equiv-is}, $\eta_{\text{Z}}$ and $\eta_{\text{R}}$ are the learning rates of ZOO and \rein{}, respectively.
}
\end{corollary}


\section{Zeroth-Order Optimization with Averaged Baseline and Query Reuse}\label{sec:zoar}

\begin{algorithm}[t]
\caption{ZOO with Averaged Baseline and Query Reuse}\label{alg:ours}
\begin{algorithmic}
\STATE {\bfseries Input:} objective function $f$, learning rate $\eta$, moment decay rates $\beta_1$, $\beta_2$, number of queries $K$ and histories $N$
\STATE {\bfseries Initialize:} $\vtheta_0, \vm_0, \vv_0, \gH_0 = \varnothing$

\FOR{iteration $t \in [T]$}
\STATE Sample $\{\rvu_k\}_{k=1}^{K}$ \\
\STATE Query $\{y_k| y_k = f(\vtheta_{t-1}{+}\mu \rvu_k; \xi)\}_{k=1}^K$ \\[1pt]
\STATE $\gH_t \leftarrow \gH_{t-1} \setminus \gH_{t-N} \cup \big\{\left(\rvu_k, f(\vtheta_{t-1}{+}\mu \rvu_k; \xi)\right)\big\}_{k=1}^K$ \\[2pt]
\STATE $\,b_t \leftarrow \frac{1}{|\gH_t|} \sum_{(\rvu, y) \in \gH_t} y$ \\
\STATE $\vg_{t} \leftarrow \frac{1}{|\gH_t|-1}\sum_{(\rvu, y) \in \gH_t} \frac{y - b_t}{\mu} \rvu$ \\[2pt]
\STATE $\vm_t \gets \beta_1 \vm_{t-1} + (1 - \beta_1) \vg_t$ \\[2pt]
\STATE $\vv_t \gets \beta_2 \vv_{t-1} + (1- \beta_2) \vm^{2}_t$\\[5pt]
\STATE $\vtheta_t \gets \vtheta_{t-1} - \eta\frac{\vm_t}{\sqrt{\vv_t + \zeta}}$
\ENDFOR
\STATE {\bfseries Output:} $\vtheta_T$
\end{algorithmic}
\end{algorithm}

Leveraging the Policy Optimization (PO) framework established in Sec.~\ref{sec:po-framework}, this section introduces \ours{} (Algo.~\ref{alg:ours}), an improved Zeroth-Order Optimization (ZOO) algorithm. We illustrate in Sec.~\ref{sec:algo-design} how \ours{} incorporates PO-inspired variance reduction techniques, including an averaged baseline and query reuse, for enhanced efficiency. While Algo.~\ref{alg:ours} demonstrates these techniques using the update rule from $\mathcal{R}$-AdaZO \citep{radazo}, their core design is general and readily adaptable to other update rules like ZO-SGD \citep{GhadimiL13a} and ZO-AdaMM \citep{zo-adamm}.
Furthermore, we provide theoretical analyses in ZOO theory to validate these PO-derived improvements in Appx.~\ref{appx:theory}.

\subsection{Algorithm Design}\label{sec:algo-design}
We introduce the two key PO-inspired techniques in \ours{} (line 5 of Algo.~\ref{alg:ours}), namely the averaged baseline and query reuse, below.

\textbf{Averaged Baseline.}
As established in Thm.~\ref{thm:gs-equiv}, the standard Gaussian-smoothed ZOO gradient estimator \eqref{eq:fd} implicitly uses $f(\vtheta; \xi)$ as a baseline, corresponding to $b(\xi) = f(\vtheta; \xi)$ in the \rein{} framework \eqref{eq:reinforce-grad-refined}. While this baseline helps reduce variance compared to no baseline, it may not be the most effective choice. In the single-step \rein{} algorithm, the baseline that minimizes the variance of the gradient estimate $\nabla \ln \pi_{\vtheta}(\rvx) (R(\rvx) - b)$ is given by $b^* = \frac{\mathbb{E}_{\rvx \sim \pi_{\vtheta}(\rvx)} \left[(\nabla \ln \pi_{\vtheta}(\rvx))^2 R(\rvx) \right]}{ \mathbb{E}_{\rvx \sim \pi_{\vtheta}(\rvx)} \left[(\nabla \ln \pi_{\vtheta}(\rvx))^2\right]}$. A simpler and widely used near-optimal baseline is the expected reward itself, $b = \mathbb{E}_{\rvx \sim \pi_{\vtheta}(\rvx)}[R(\rvx)]$.
In our ZOO context, where $R(\rvx) = -F(\rvx) = -\mathbb{E}_{\xi}[f(\rvx; \xi)]$ and $\rvx = \vtheta + \mu \rvu$, this corresponds to choosing the baseline as $b = \mathbb{E}_{\rvx \sim \pi_{\vtheta}(\rvx)}[F(\rvx)] = F_{\mu}(\vtheta)$.
The standard ZOO baseline $f(\vtheta; \xi)$ can be seen as a single-sample, zero-order approximation of $F_{\mu}(\vtheta)$ evaluated at the center point.
Algo.~\ref{alg:ours} proposes using a more robust estimate of this expected value. Specifically, it computes the baseline $b_t$ as the empirical average of function values obtained from recent queries stored in a history buffer $\gH_t$:
\begin{equation}
b_t \triangleq \frac{1}{|\gH_t|} \sum_{(\rvu, y) \in \gH_t} y \ , \label{eq:baseline-avg}
\end{equation}
where $y = f(\vtheta_{t'} + \mu \rvu; \xi)$ for some past iteration $t' \le t-1$. This average in fact serves as a Monte Carlo estimate of the expected function value $F_{\mu}(\vtheta)$, potentially providing a lower-variance baseline compared to the single evaluation $f(\vtheta; \xi)$ used implicitly in \eqref{eq:fd}, which we will verify in Appx.~\ref{appx:theory}.

\textbf{Query Reuse.}
To further enhance sample efficiency and reduce variance, Algo.~\ref{alg:ours} incorporates query reuse. This mirrors the concept of using off-policy data, common in algorithms like Proximal Policy Optimization (PPO) \citep{schulman2017proximal}, where experiences gathered under previous policies are reused to improve the current policy update, thereby increasing data efficiency. In our ZOO context, Algo.~\ref{alg:ours} maintains a history buffer $\gH_t$ containing the $N\times K$ most recent query results (pairs of perturbation vectors $\rvu$ and corresponding function values $y$). At iteration $t$, $K$ new queries based on $\vtheta_{t-1}$ are performed, added to the buffer, and the oldest $K$ queries are discarded. Crucially, the gradient estimate $\vg_t = \hat{\nabla} F(\vtheta_{t-1})$ is then computed using all samples currently in the history $\gH_t$:
\begin{equation}
\hat{\nabla} F(\vtheta_{t-1}) \triangleq \frac{1}{|\gH_t| - 1}\sum_{(\rvu, y) \in \gH_t} \frac{y - b_t}{\mu} \rvu \ . \label{eq:grad-avg}
\end{equation}
This approach uses all $|\gH_t| = N \times K$ samples, effectively increasing the gradient estimation batch size without additional queries beyond the initial $K$. The resulting averaging over a larger set is expected to produce a gradient estimate with significantly lower variance (verified in Appx.~\ref{appx:theory}).

\begin{figure*}[t]
\vspace{-2mm}
    \centering
    \includegraphics[width=0.95\textwidth]{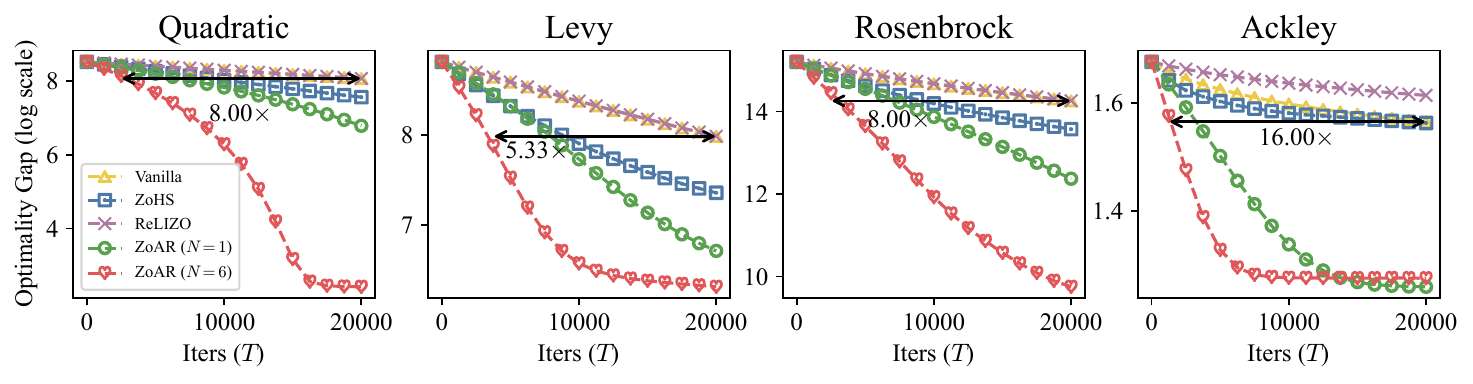}
    \vspace{-3mm}
    \caption{Comparison of convergence among different zeroth-order optimization algorithms on four synthetic functions. 
    All curves are averaged over 5 independent runs.}
    \label{fig:synthetic_adam}
    \vspace{-4mm}
\end{figure*}

\textbf{Advantages.}
The proposed \ours{} algorithm offers several compelling advantages. \textit{(a)} It provides significant \textit{variance reduction} by employing an averaged baseline $b_t$ and reusing historical queries from $\gH_t$ (see Appx.~\ref{appx:theory}) compared to conventional ZOO with finite difference \citep{Nesterov2017}. \textit{(b)} Compared to \citep{prgf, relizo}, the algorithm maintains compelling \textit{computational and memory efficiency}, as the overhead for managing the history buffer (using only random seeds like \citep{mezo, ferret}) and performing the averaging calculations is generally modest, which is scaling linearly with history size. \textit{(c)} \ours{} benefits from \textit{ease of Implementation}, representing a straightforward modification to standard ZOO procedures by incorporating a buffer and simple averaging steps. \textit{(d)} It offers enhanced \textit{sample efficiency and flexibility} by leveraging the accumulated information in $\gH_t$: a meaningful gradient estimate $\vg_t$ can be computed even if only a small number of new queries (potentially $K=1$) are performed at each iteration. These advantages make \ours{} a practical approach for improving ZOO performance, particularly in optimization settings where variance control and query efficiency is crucial.

\section{Experiments}\label{sec:main-exp}
In this section, we conduct extensive experiments on synthetic functions (Sec. \ref{sec:synthetic}) and black-box adversarial attack (Sec. \ref{sec:adversarial}). 
More results, e.g., the equivalence between ZOO and \rein{}, memory-efficient LLM fine-tuning, are in Appx. \ref{sec:additional_experiments}.

\subsection{Synthetic Functions}\label{sec:synthetic}

\begin{table}[t]
\vspace{-2mm}
\caption{Comparison of the minimal number of iterations to achieve a successful attack for different ZOO methods. Results are averaged over 5 runs. The speedup is compared against the Vanilla ZOO.}
\label{tab:adversarial-radazo}
\centering
\resizebox{\columnwidth}{!}{
\begin{tabular}{@{}llcccc@{}}
\toprule
&   Metric    & Vanilla         & ZoHS           & ZoAR w/o history & ZoAR           \\ \midrule
\multirow{2}{*}{$\mathcal{R}$-AdaZO} & \# Iters $(\times 10^2)$ & 23.3$\pm$5.4  & 23.3$\pm$2.6 & 12.4$\pm$1.0   & \textbf{8.56}$\pm$2.2 \\
                                        & Speedup                  & 1.0$\times$    & 1.0$\times$   & 1.87$\times$    & \textbf{2.72}$\times$  \\ \midrule
\multirow{2}{*}{ZO-AdaMM}     & \# Iters $(\times 10^2)$ & \;\;65.3$\pm$12.7 & 36.9$\pm$5.8 & 32.6$\pm$8.0   & \textbf{11.0}$\pm$2.8 \\
                                        & Speedup                  & 1.0$\times$    & 1.8$\times$   & 2.0$\times$     & \textbf{5.92}$\times$  \\ \bottomrule
\end{tabular}
}
\vspace{-2mm}
\end{table}

\textbf{The Superiority of \ours{}.} We subsequently evaluate the convergence rate and final performance of \ours{} against several baselines on four synthetic functions of dimensionality $d=10^4$ (detailed in Appx.~\ref{subsec:synthetic_details}). The compared methods include Vanilla ZOO \citep{Nesterov2017}, ReLIZO \citep{relizo}, and ZoHS (details in Appx.~\ref{appx:baselines}). Fig.~\ref{fig:synthetic_adam} presents the results using the ZO-AdaMM \citep{zo-adamm} update rule, while corresponding results under the $\mathcal{R}$-AdaZO \citep{radazo} update rule are available in Appx.~\ref{subsec:synthetic_radazo}. The results in Fig.~\ref{fig:synthetic_adam} show that \ours{} consistently outperforms all baseline algorithms in both convergence speed and final optimization performance. Notably, \ours{} with $N=6$ achieves an $8\times$ speedup over Vanilla ZOO on the Quadratic and Rosenbrock functions, and a $16\times$ speedup on the Ackley function. Moreover, comparing \ours{} with $N=6$ (utilizing query reuse) against \ours{} with $N=1$ (using only the averaged baseline) illustrates the significant additional benefit of 
historical information.

\subsection{Black-box Adversarial Attack}\label{sec:adversarial}
We further evaluate the performance of \ours{} in the domain of black-box adversarial attacks, a prominent application of zeroth-order optimization \citep{prgf, zord}. In this scenario, the goal is to identify an optimal perturbation $\boldsymbol{\delta}$ for a given input image $x$ such that a target black-box model misclassifies $x+\boldsymbol{\delta}$. Our experimental setup follows that introduced by \citep{radazo}, targeting a convolutional neural network (CNN) trained on the MNIST dataset \citep{lecun1998mnist} (more details in Appx.~\ref{subsec:adversarial_details}). We assess algorithm efficiency by the minimum number of iterations required to achieve a successful attack. The comparison includes Vanilla ZOO and ZoHS, with each evaluated under both the ZO-AdaMM \citep{zo-adamm} and $\mathcal{R}$-AdaZO \citep{radazo} update rules. ReLIZO is omitted from this comparison as it failed to achieve a successful attack within the maximum query budget. The results are summarized in Tab.~\ref{tab:adversarial-radazo}, showing that \ours{} achieves the fastest attack success across both update rules. Specifically, under the ZO-AdaMM setting, \ours{} represents a $5.92\times$ speedup compared to Vanilla ZOO. The less pronounced speedup of \ours{} with $\mathcal{R}$-AdaZO (versus ZO-AdaMM) is likely due to the inherent gradient variance reduction of $\mathcal{R}$-AdaZO \citep{radazo}, which may reduce the marginal impact of additional variance mitigation from \ours{}.

\section{Conclusion}
This paper established a novel and fundamental equivalence between zeroth-order optimization (ZOO) with finite differences and single-step policy optimization (PO). Leveraging this PO framework, we introduced \ours{}, an algorithm incorporating PO-inspired variance reduction techniques (an averaged baseline and query reuse) that demonstrably enhance performance. Our theoretical and empirical results highlight the benefits of this unified perspective, offering new insights into ZOO and providing a principled path for future algorithmic advancements.


%% file: workspace/appendix.tex
\appendix
\onecolumn
\section{Related Work}\label{sec:related}
Zeroth-Order (ZO) optimization research has primarily advanced along two interconnected fronts: the design of gradient estimators and the development of update rules or full algorithms.

\textbf{ZO Gradient Estimation.} A cornerstone of ZOO is the estimation of gradients using only function evaluations, typically through finite difference approximations. Seminal work introduced Gaussian random perturbations for smooth objectives, establishing theoretical convergence \citep{Nesterov2017}. Other perturbation strategies include uniform sampling from the unit sphere \citep{bsg} or coordinate-wise perturbations \citep{coordinate}. A primary challenge with these methods is the high variance in their gradient estimates. To address this, several approaches have been developed. E.g., prior-guided gradient estimation leverages historical estimates to denoise current ones \citep{prgf}. Recently, methods have explored learning surrogate models of the objective function using past queries to derive more stable gradient estimates \citep{zord, fzoos}. Another line of work has focused on linear interpolation strategies for more accurate estimates by reusing queries from prior iterations to reduce complexity while maintaining sample quality \citep{relizo}. While these methods offer valuable improvements, the underlying connection between the widely-used finite difference ZOO gradient estimators and principles from Reinforcement Learning (RL), particularly Policy Optimization (PO), has remained largely unelucidated. Our work bridges this gap by reinterpreting these estimators through a PO lens, which not only reveals inherent variance reduction mechanisms but also inspires new ones. Leveraging this novel PO framework, this paper introduces new PO-inspired variance reduction techniques, specifically an averaged baseline and query reuse, which are central to our proposed \ours{} algorithm and aim to significantly improve the stability and efficiency of ZO gradient estimation.

\textbf{ZO Update Rules and Algorithms.} Given a ZO gradient estimate, many ZOO algorithms directly adopt update rules from first-order (FO) optimization. A significant body of work employs Stochastic Gradient Descent (SGD) or its variants \citep{GhadimiL13a, GhadimiLZ16, Nesterov2017, 0001LCHA18, 0001KCTCA18, prgf, zord}. Recognizing the potential benefits of adaptive step sizes, some research has integrated adaptive methods like Adam \citep{adam} into the ZOO setting \citep{zo-adamm, nazari2020adaptive, adamu}. Further advancing these adaptive methods, recent work such as $\gR$-AdaZO \citep{radazo} has focused on refining the utilization of moment information, demonstrating how careful handling of first and second moment estimates can lead to significant variance reduction in the gradient estimates and a more accurate capture of the optimization landscape, thereby improving convergence. Notably,  this paper does not aim to introduce a new update rule, but focus on unveiling the fundamental connection between ZOO and PO, and developing advanced gradient estimation method that is applicable to all these existing update rules and algorithms. 

\section{Theoretical Analysis}\label{appx:theory}
This section provides a theoretical underpinning for our \ours{} (Algo.~\ref{alg:ours}). We analyze the bias of its gradient estimator, the optimality of its baseline, the bias-variance trade-off, and its convergence. To ease our proof, we follow the common practice in \citep{radazo} to prove under $\rvu \sim \mathrm{Unif}(\sS^{d-1})$ and the following commonly used assumptions.
\begin{assumption}[Bounded Continuity and Smoothness]\label{assump:1}
\textit{\fontfamily{ppl}\selectfont
$\forall \vtheta, \vtheta' \in \sR^d$ and $i \in [d]$,
\begin{equation}
\begin{aligned}
|f(\vtheta, \xi)| &\leq C \ , \\
|F(\vtheta) - F(\vtheta')| &\leq  L_0\|\vtheta - \vtheta'\| \ , \\
|\nabla_i F(\vtheta) - \nabla_i F(\vtheta')| &\leq L_1 \|\vtheta - \vtheta'\| \ .
\end{aligned}
\end{equation}
}
\end{assumption}

\begin{assumption}[Bounded Variance]\label{assump:2}
\vspace{-1mm}
\textit{\fontfamily{ppl}\selectfont
$\forall \vtheta \in \sR^d$,
\begin{equation}
\begin{aligned}
\E_{\xi}[| f(\vtheta, \xi) - F(\vtheta)|^2] &\leq \sigma_{\xi}^2 \ ,\\
\E_{\rvu}[|F(\vtheta + \mu \rvu) - F_\mu(\vtheta)|^2] &\leq \sigma_{\mu}^2 \ .
\end{aligned}
\end{equation}
}
\vspace{-3mm}
\end{assumption}

\begin{theorem}[Bias Analysis]\label{thm:bias-analysis}
\textit{\fontfamily{ppl}\selectfont
For every iteration $t$ of \ours{} (Algo.~\ref{alg:ours}) with history depth $N\geq1$ and $K$ queries per step, the expected value of the gradient estimator $\hat{\nabla} F(\vtheta_{t-1})$ is:
\begin{equation*}
    \E\left[\hat{\nabla} F(\vtheta_{t-1})\right] = \frac{1}{N}\sum_{n=1}^N \nabla F_{\mu}(\vtheta_{t-n}) \ .
\end{equation*}
}
\vspace{-4mm}
\end{theorem}
\textbf{Remark.} Its proof is in Appx.~\ref{appx:proof-of-bias-analysis}. Thm.~\ref{thm:bias-analysis} reveals that $\hat{\nabla} F(\vtheta_{t-1})$ in \eqref{eq:grad-avg} of \ours{} is secretly an unbiased estimator for the average of smoothed gradients from the current and $N-1$ previous parameters. This implies that \eqref{eq:grad-avg} implicitly targets this historically averaged groundtruth, a mechanism that shall potentially reduce the gradient estimation variance at  $\vtheta_{t-1}$ by effectively increasing the number of queries contributing to the estimate (see Thm.~\ref{thm:bias-variance-decomp}). However, \eqref{eq:grad-avg} is biased with respect to the current smoothed gradient $\nabla F_{\mu}(\vtheta_{t-1})$ when $N \ge 2$, which emerges because these historical parameters $\vtheta_{t-n}$ have diverged from $\vtheta_{t-1}$. This is an inherent consequence of leveraging historical queries for variance reduction, creating a bias-variance trade-off detailed in Thm.~\ref{thm:bias-variance-decomp}. Notably, if $N=1$ (no query reuse beyond the current batch), the estimator becomes unbiased for $\nabla F_{\mu}(\vtheta_{t-1})$.


\begin{theorem}[Optimal Baseline]\label{thm:opt-baseline}
\textit{\fontfamily{ppl}\selectfont
Let $\rvu \sim \mathrm{Unif}(\sS^{d-1})$, for every $t$ of \ours{} (Algo.~\ref{alg:ours}) with $N\geq1$, the optimal $b_t$ to minimize  $\Var\left(\hat{\nabla} F(\vtheta_{t-1})\right) = \E\left[\left\|\hat{\nabla} F(\vtheta_{t-1}) - \frac{1}{N}\sum_{n=1}^N \nabla F_{\mu}(\vtheta_{t-n})\right\|^2\right] $ is
\vspace{-3mm}
\begin{equation*}
b_t^* = \frac{1}{N}\sum_{n=1}^N F_{\mu}(\vtheta_{t-n}) \ .
\end{equation*}
}
\vspace{-5mm}
\end{theorem}
\textbf{Remark.} Its proof is in Appx.~\ref{appx:proof-of-opt-baseline}. Thm.~\ref{thm:opt-baseline} provides strong theoretical support for the averaged baseline in \ours{}. It demonstrates that for gradient estimator \eqref{eq:grad-avg} of \ours{} under $\rvu \sim \mathrm{Unif}(\sS^{d-1})$, the baseline $b_t$ defined in \eqref{eq:baseline-avg} is in fact a practical Monte Carlo approximation of the variance-minimizing $b_t^*$. This result formalizes the intuition that averaging recent function evaluations provides a more effective baseline than a single point estimate (like $f(\vtheta;\xi)$ implicitly used in vanilla ZOO, or no baseline at all), thereby contributing to the overall variance reduction of the gradient estimate from a pure perspective of ZOO theory. Crucially, the structural similarity between the optimal $b_t^*$ above and the variance-minimizing baseline $b = \mathbb{E}_{\rvx \sim \pi_{\vtheta}(\rvx)}[R(\rvx)]$ in the \rein{} algorithm further underscores the principled PO foundation and validates the practical efficacy of our $b_t$ approximation.

\begin{theorem}[Bias-Variance Decomposition]\label{thm:bias-variance-decomp}
\textit{\fontfamily{ppl}\selectfont
Let $\rvu \sim \mathrm{Unif}(\sS^{d-1})$ and $b_t$ in \eqref{eq:grad-avg} be the optimal $b_t^*$ in Thm.~\ref{thm:opt-baseline}, under Assump.~\ref{assump:1} and \ref{assump:2}, for every $t$ of \ours{} (Algo.~\ref{alg:ours}) with $N\geq1$,
\begin{equation*}
\begin{aligned}
    &\E\left[\left\|\hat{\nabla} F(\vtheta_{t-1}) - \nabla F_{\mu}(\vtheta_{t-1})\right\|^2\right] \leq \underbrace{\frac{\sigma_{\xi}^2 + \sigma_{\mu}^2}{N K \mu^2}}_{\text{\normalfont Variance $\triangleq V$}} + \underbrace{\frac{\eta^2 L_0^2 d \qty(N^2 - 1)}{3 \qty(1 - \beta_2) N^2 K \mu^2} + \frac{\eta^2 L_1^2 d^2 \qty(N - 1)}{2 \qty(1 - \beta_2)}}_{\text{\normalfont Squared Bias} \phantom{\triangleq}} \ .
\end{aligned}
\end{equation*}
}
\vspace{-2mm}
\end{theorem}
\textbf{Remark.} Its proof is in Appx.~\ref{appx:proof-of-bias-variance-decomp}. Thm.~\ref{thm:bias-variance-decomp} explicitly quantifies the fundamental trade-off inherent in the query reuse mechanism of \ours{}. The first component \textit{variance} $V$ is what \ours{} primarily targets for reduction through its PO-inspired techniques: the averaged baseline (justified by Thm.~\ref{thm:opt-baseline}) and the reuse of historical samples. The second component \textit{squared bias} arises because the gradient estimator, as shown in Thm.~\ref{thm:bias-analysis}, is an average of historical smoothed gradients, which may differ from the current target gradient $\nabla F_\mu(\vtheta_{t-1})$. 
Consequently, while increasing the history depth $N$ can substantially decrease variance, it may simultaneously inflate the bias. Fortunately, this bias can be small with a small learning rate $\eta$, and is completely avoided when $N=1$.

\begin{theorem}[Variance-Aware Convergence, Informal]\label{thm:conv-zoo}
\textit{\fontfamily{ppl}\selectfont
Let $\rvu \sim \mathrm{Unif}(\sS^{d-1})$ and $b_t$ in \eqref{eq:grad-avg} be the optimal $b_t^*$ in Thm.~\ref{thm:opt-baseline}, under Assump.~\ref{assump:1} and \ref{assump:2}, when $1 - \beta_2 \sim \gO\qty(\epsilon^2)$, $\eta \sim \gO\qty(\epsilon^2)$, $T \sim \gO\qty(\epsilon^{-4})$,  $\beta_1 \leq \sqrt{\beta_2}, \beta_2 > 1 / 2, \vm_{0,i}\,{=}\,0,\vv_{0,i}>0\,(\forall{i}\in[d])$, the following holds for \ours{} (Algo.~\ref{alg:ours}) under certain constants $B_1$ and $B_2$ that are independent of $\epsilon$,
    \begin{equation} 
    \begin{aligned}
        \frac1T \sum_{t=1}^T \E[\|\nabla F(\vtheta_t)\|] \leq& \sqrt{\frac{2}{\beta_1 \qty(1 - \beta_2)}} \qty(1 + \beta_1) \epsilon^2 + \qty(\sqrt[4]{\zeta} + \sqrt{\Xi}) \epsilon  + \mu L_1 \sqrt{d} + B_2 \ , \nonumber
    \end{aligned}
    \end{equation}
    where $\Xi \triangleq B_1 + \sqrt{\frac{2 \qty(1 + \beta_1) \qty(1 - \beta_1)^2 \beta_2}{\qty(\beta_2 - \beta_1^2) \qty(1 - \beta_2)}V}$ and $V = \frac{\sigma_{\xi}^2 + \sigma_{\mu}^2}{N K \mu^2}$.
}
\vspace{-3mm}
\end{theorem}
\textbf{Remark.} Its proof is in Appx.~\ref{appx:proof-of-conv}. Thm. \ref{thm:conv-zoo} presents the convergence of \ours{} (Algo.\ref{alg:ours}). Notably, it highlights that the convergence rate is directly influenced by the variance $V$ (occurred in $\Xi$) of our gradient estimator \eqref{eq:grad-avg}, which corresponds to the variance term in Thm.~\ref{thm:bias-variance-decomp}. When $V$ dominates the convergence, minimizing $V$ is hence the key to achieving better optimization performance, which provides a strong theoretical support for the variance reduction techniques used in \ours{} (averaged baseline and query reuse). The $\mathcal{O}(\mu)$ term is standard in conventional ZOO, reflecting the inherent discrepancy from optimizing the smoothed objective $F_{\mu}$ instead of $F$. While the additional $B_2$ term results from the bias introduced by our \eqref{eq:grad-avg} as revealed in Thm.~\ref{thm:bias-analysis}, it can be small with a small $\eta$.

\section{Proofs}\label{sec:proof}
\subsection{Proof of Thm.~\ref{thm:obj-equiv}}\label{appx:proof-of-obj-equiv}
\begin{proof}
By definition, $J(\vtheta) = \mathbb{E}_{\rvx \sim \pi_{\vtheta}(\rvx)}[F(\rvx)]$. Substituting the reparameterization $\rvx = \vtheta + \mu \rvu$ where $\rvu \sim p(\rvu)$, the expectation over $\rvx \sim \pi_{\vtheta}(\rvx)$ becomes an expectation over $\rvu \sim p(\rvu)$:
\begin{equation}
J(\vtheta) = \mathbb{E}_{\rvu \sim p(\rvu)}\left[F(\vtheta + \mu \rvu)\right] \ ,
\end{equation}
which is precisely the definition of the ZOO smoothed objective $F_{\mu}(\vtheta)$ in \eqref{eq:fu}.
\end{proof}

\subsection{Proof of Thm.~\ref{thm:gs-equiv}}\label{appx:proof-of-gs-equiv}
\begin{proof}
Substituting $\rvx_k = \vtheta + \mu \rvu_k$ and $b(\xi) = f(\vtheta; \xi)$ into the \rein{} estimator \eqref{eq:reinforce-grad-refined}:
\begin{align}
\hat{\nabla}_{\text{GS}} J(\vtheta) = \frac{1}{K}\sum_{k=1}^{K} \frac{(\vtheta + \mu \rvu_k) - \vtheta}{\mu^2} \left(f(\vtheta + \mu \rvu_k; \xi) - f(\vtheta; \xi)\right) = \frac{1}{K}\sum_{k=1}^{K} \frac{f(\vtheta + \mu \rvu_k; \xi) - f(\vtheta; \xi)}{\mu} \rvu_k \ .
\end{align}
This is exactly the Gaussian-smoothed ZOO gradient estimator $\hat{\nabla} F(\vtheta)$ in \eqref{eq:fd}.
\end{proof}

\subsection{Proof of Thm.~\ref{thm:zoo-equiv-is}}\label{appx:proof-of-zoo-equiv-is}
\begin{proof}


Initially, inserting $\rvx_k = \vtheta + \mu \rvu_k$ into $\pi_{\vtheta}(\rvx_k) = \gN(\vtheta, \mu^2 \rmI_d)$:
\begin{equation}
    \pi_{\vtheta}(\rvx_k) = \frac{e^{-\frac{\norm{\rvx_k - \vtheta}}{2 \mu^2}}}{\qty(2 \pi \mu^2)^{\frac{d}{2}}} = \frac{1}{\mu^d} \frac{e^{-\frac{\norm{\rvu_k}}{2}}}{\qty(2 \pi)^{\frac{d}{2}}} \ .
\end{equation}

We next consider the transformation of $p(\rvx_k)$ under three different distributions separately:
\begin{itemize}
[topsep=0pt,leftmargin=12mm,itemsep=0pt]
    \item[\normalfont (I)] If $\vu_k$ follows the standard Gaussian distribution, $\vx_k$ follows the Gaussian distribution $\gN(\vtheta, \mu^2 \rmI_d)$, then consequently:
    \begin{equation}
        p(\rvx_k) = \frac{e^{-\frac{\norm{\rvx_k - \vtheta}}{2 \mu^2}}}{\qty(2 \pi \mu^2)^{\frac{d}{2}}} = \frac{1}{\mu^d} \frac{e^{-\frac{\norm{\rvu_k}}{2}}}{\qty(2 \pi)^{\frac{d}{2}}} \ ,
    \end{equation}
    which is the same as $\pi_{\vtheta}(\rvx_k)$.
    \item[\normalfont (II)] If $\vu_k$ follows the uniform distribution over the unit sphere $\mathrm{Unif}(\sS^{d-1})$, $\rvx_k$ follows the uniform distribution over the sphere with radius $\mu$ $\mathrm{Unif}(\sS^{d-1}(\mu))$, then consequently:
    \begin{equation}
        p(\rvx_k) = \frac{1}{\text{Area}\qty(\sS^{d-1}(\mu))} = \frac{\Gamma\qty(\frac{d}{2})}{2 \pi^{\frac{d}{2}} \mu^{d - 1}} \ ,
    \end{equation}
    with constraint $\norm{\rvu} = 1$, where $\Gamma(\cdot)$ denotes Gamma function.
    \item[\normalfont (III)] If $\vu_k$ follows the uniform distribution over standard basis vectors $\mathrm{Unif}(\{\ve_i\}_{i = 1}^d)$, $\rvx_k$ follows the uniform distribution over the orthonormal basis vectors $\mathrm{Unif}(\{\vtheta + \mu \ve_i\}_{i = 1}^d)$, then consequently:
    \begin{equation}
        p(\rvx_k) = \frac{1}{d} \ ,
    \end{equation}
    with constraint $\rvu \in \{\ve_i\}_{i = 1}^d$.
\end{itemize}
Afterthat, let $\gamma \triangleq \frac{\pi_{\vtheta}(\rvx_k)}{p(\rvx_k)}$, and substitute $\rvx_k = \vtheta + \mu \rvu_k$, $b(\xi) = f(\vtheta; \xi)$ into IS-based \rein{} gradient estimator \eqref{eq:reinforce-grad-is}:
\begin{equation}
    \hat{\nabla}_{\text{IS}} J(\vtheta) = \gamma \frac{1}{K}\sum_{k=1}^{K} \frac{f(\vtheta + \mu \rvu_k; \xi) - f(\vtheta; \xi)}{\mu} \rvu_k = \gamma \hat{\nabla} F(\vtheta) \ ;
\end{equation}
where \normalfont{(I)} for $\rvu_k \sim \gN(\vzero, \mu^2 \rmI_d)$, $\gamma = 1$; \normalfont{(II)} for $\rvu_k \sim \mathrm{Unif}(\sS^{d-1})$, $\gamma = \frac{2^{1 - \frac{d}{2}} e^{- \half}}{\mu \Gamma(\frac{d}{2})}$; \normalfont{(III)} for $\rvu_k \sim \mathrm{Unif}(\{\ve_i\}_{i = 1}^d)$, $\gamma = \frac{d e^{- \half}}{\qty(2 \pi \mu^2)^{\frac{d}{2}}}$.

\end{proof}

\subsection{Proof of Thm. \ref{thm:bias-analysis}}\label{appx:proof-of-bias-analysis}

\begin{proof}
By inserting the average baseline \eqref{eq:baseline-avg}, the expectation of the gradient estimator \eqref{eq:grad-avg} can be expressed as follows:
\begin{equation}
\begin{aligned}
    \E\qty[\hat{\nabla} F(\vtheta_{t - 1})] =& \frac{1}{N K - 1} \E_{\rvu} \E_{\xi} \qty[\sum_{n, k = 1}^{N, K} \frac{f\qty(\vtheta_{t - n} + \mu \rvu_{t - n, k}; \xi) - \frac{1}{N K} \sum_{n', k' = 1}^{N, K} f\qty(\vtheta_{t - n} + \mu \rvu_{t - n', k'}; \xi)}{\mu} \rvu_{t - n, k}] \\
    =& \frac{1}{N K - 1} \E_{\rvu} \qty[\sum_{n, k = 1}^{N, K} \frac{F\qty(\vtheta_{t - n} + \mu \rvu_{t - n, k}) - \frac{1}{N K} \sum_{n', k' = 1}^{N, K} F\qty(\vtheta_{t - n} + \mu \rvu_{t - n', k'})}{\mu} \rvu_{t - n, k}] \\
    =& \frac{1}{N K - 1} \E_{\rvu} \qty[\sum_{n, k = 1}^{N, K} \frac{\frac{N K - 1}{N K} F\qty(\vtheta_{t - n} + \mu \rvu_{t - n, k}) - \frac{1}{N K} \sum_{
        \substack{n', k' = 1 \\
    n' \neq n, k' \neq k}}^{N, K} F\qty(\vtheta_{t - n} + \mu \rvu_{t - n', k'})}{\mu} \rvu_{t - n, k}] \\
    \overset{(a)}{=}& \frac{1}{N K} \sum_{n, k = 1}^{N, K} \E_{\rvu_{t - n, k}} \qty[\frac{F\qty(\vtheta_{t - n} + \mu \rvu_{t - n, k})}{\mu} \rvu_{t - n, k}] \\
    \overset{(b)}{=}& \frac{1}{N K} \sum_{n, k = 1}^{N, K} \E_{\rvu_{t - n, k}} \qty[\nabla F\qty(\vtheta_{t - n} + \mu \rvu_{t - n, k})] \\
    =& \frac{1}{N} \sum_{n = 1}^{N} \nabla F_{\mu} \qty(\vtheta_{t - n}) \ ,
\end{aligned}
\end{equation}
where $\E_{\rvu}$ denote the expectation over $\rvu_{t - 1, 1}, \cdots, \rvu_{t - n, K}$ for simplicity. Besides, $(a)$ follows from the fact that $\rvu_{t - n', k'}$ within the summation over $n', k'$ is uncorrelated with $\rvu_{t - n, k}$. When $\rvu \sim \gN(\vzero, \rmI^d)$, $(b)$ is a direct consequence of Stein's Lemma in \citep{stein1981estimation}. Alternatively, when $\rvu \sim \mathrm{Unif}(\sS^{d-1})$, $(b)$ follows from Lemma 2.1 in \citep{flaxman2005online}, utilizing the definition $F_\mu(\vtheta) \triangleq \E_{\rvu \sim \mathrm{Unif}(\sB^d)} \left[F(\vtheta + \mu \rvu)\right]$.
\end{proof}
\textbf{Remark.} Note that in the step $(a)$, the baseline term $\frac{1}{N K} \sum_{\substack{n', k' = 1 \\
n' \neq n, k' \neq k}}^{N, K} F\qty(\vtheta_{t - n} + \mu \rvu_{t - n', k'})$ vinishes because it is independent of the random variable $\rvu_{t - n, k}$. 
Similar to the \rein{} estimator in \eqref{eq:pg-zo-obj}, incorporating the baseline $b(\xi)$ does not alter the expected value of the gradient estimator, which further supports the connection between the \rein{} and ZOO gradient estimators.

\subsection{Proof of Thm. \ref{thm:opt-baseline}} \label{appx:proof-of-opt-baseline}
\begin{proof}
Beginning with the definition of $\Var\qty(\hat{\nabla} F(\vtheta_{t-1}))$:
\begin{equation}\label{eq:f_var}
\begin{aligned}
    &\Var\qty(\hat{\nabla} F(\vtheta_{t-1})) = \E\qty[\norm{\hat{\nabla} F(\vtheta_{t-1}) - \frac{1}{N} \sum_{n = 1}^N \nabla F_\mu \qty(\vtheta_{t - n})}^2] \\
    =& \E\qty[\norm{\hat{\nabla} F(\vtheta_{t-1})}^2] - 2 \aqty{\E\qty[\hat{\nabla} F(\vtheta_{t-1})], \frac{1}{N} \sum_{n = 1}^N \nabla F_\mu \qty(\vtheta_{t - n})} + \norm{\frac{1}{N} \sum_{n = 1}^N \nabla F_\mu \qty(\vtheta_{t - n})}^2 \\
    \overset{(a)}{=}& \E\qty[\norm{\frac{1}{N K} \sum_{n, k = 1}^{N, K} \frac{f(\vtheta_{t - n} + \mu \rvu_{t - n, k}; \xi) - b_t}{\mu} \rvu_{t - n, k}}^2] - \norm{\frac{1}{N} \sum_{n = 1}^N \nabla F_\mu \qty(\vtheta_{t - n})}^2 \ ,
\end{aligned}
\end{equation}
where $(a)$ comes from Thm. \ref{thm:bias-analysis}.

It is obvious that \eqref{eq:f_var} is actually a quadratic function w.r.t $b_t$, and hence the optimal $b_t^*$ is derived by setting its derivative to zero $\frac{\partial }{\partial b_t} \Var\qty(\hat{\nabla} F(\vtheta_{t-1})) = 0$, i.e.
\begin{equation}
\begin{aligned}
    \frac{\partial }{\partial b_t} \Var\qty(\hat{\nabla} F(\vtheta_{t-1})) =& 2 \E\qty[\aqty{\frac{1}{N K} \sum_{n, k = 1}^{N, K} \frac{f(\vtheta_{t - n} + \mu \rvu_{t - n, k}; \xi) - b_t}{\mu} \rvu_{t - n, k}, - \frac{1}{N K \mu} \sum_{n, k = 1}^{N, K} \rvu_{t - n, k}}] \\
    \overset{(a)}{=}& - \frac{2}{N^2 K^2 \mu^2} \sum_{n, k = 1}^{N, K} \E\qty[\aqty{\qty(f(\vtheta_{t - n} + \mu \rvu_{t - n, k}; \xi) - b_t) \rvu_{t - n, k}, \rvu_{t - n, k}}] \\
    =& - \frac{2}{N^2 K^2 \mu^2} \sum_{n, k = 1}^{N, K} \E_{\rvu_{t-n, k}}\qty[\aqty{\qty(F(\vtheta_{t - n} + \mu \rvu_{t - n, k}) - b_t) \rvu_{t - n, k}, \rvu_{t - n, k}}] \\
    =& - \frac{2}{N^2 K^2 \mu^2} \sum_{n, k = 1}^{N, K} \qty(\E_{\rvu_{t-n, k}}\qty[F(\vtheta_{t - n} + \mu \rvu_{t - n, k}) \norm{\rvu_{t - n, k}}^2] - b_t \E_{\rvu_{t-n, k}}\qty[\norm{\rvu_{t - n, k}}^2]) \ ,
\end{aligned}
\end{equation}
where $(a)$ is due to the independence of $\rvu_{t - n, k}$ across different iterations $t$ and queries $k$.

Setting $\frac{\partial }{\partial b_t} \Var\qty(\hat{\nabla} F(\vtheta_{t-1})) = 0$, we have:
\begin{equation}
\begin{aligned}
    b_t^* =& \frac{\frac{1}{N K} \sum_{n, k = 1}^{N, K} \E_{\rvu_{t-n, k}}\qty[F(\vtheta_{t - n} + \mu \rvu_{t - n, k}) \norm{\rvu_{t - n, k}}^2]}{\frac{1}{N K} \sum_{n, k = 1}^{N, K} \E_{\rvu_{t-n, k}}\qty[\norm{\rvu_{t - n, k}}^2]} = \frac{\frac{1}{N} \sum_{n = 1}^N \E_{\rvu}\qty[F(\vtheta_{t - n} + \mu \rvu) \norm{\rvu}^2]}{\E_{\rvu}\qty[\norm{\rvu}^2]} \ .
\end{aligned}
\end{equation}
If $\rvu_k \sim \mathrm{Unif}(\sS^{d-1})$ or $\rvu_k \sim \mathrm{Unif}(\{\ve_i\}_{i = 1}^d)$, it follows that $\norm{\rvu}^2 = 1$. Consequently:
\begin{equation}
    b_t^* = \frac{1}{N} \sum_{n = 1}^N \E_{\rvu}\qty[F(\vtheta_{t - n} + \mu \rvu)] = \frac{1}{N} \sum_{n = 1}^N F_\mu \qty(\vtheta_{t - n}) \ ,
\end{equation}
which completes the proof.
\end{proof}

\subsection{Proof of Thm. \ref{thm:bias-variance-decomp}}\label{appx:proof-of-bias-variance-decomp}
\begin{proof}
To begin with, we let $b_t$ be the optimal value obtained from Thm. \ref{thm:opt-baseline}, and proceed with the calculation of $\Var\qty(\hat{\nabla} F(\vtheta_{t-1}))$ from \eqref{eq:f_var}:
\begin{equation}\label{eq:f_var_1}
    \Var\qty(\hat{\nabla} F(\vtheta_{t-1})) \leq \E\qty[\norm{\frac{1}{N K} \sum_{n, k = 1}^{N, K} \frac{f(\vtheta_{t - n} + \mu \rvu_{t - n, k}; \xi) - b_t}{\mu} \rvu_{t - n, k}}^2] - \norm{\frac{1}{N} \sum_{n = 1}^N \nabla F_\mu \qty(\vtheta_{t - n})}^2 \ .
\end{equation}

For the first term of \eqref{eq:f_var_1}:
\begin{equation}\label{eq:f_var_2}
\begin{aligned}
    &\E\qty[\norm{\frac{1}{N K} \sum_{n, k = 1}^{N, K} \frac{f(\vtheta_{t - n} + \mu \rvu_{t - n, k}; \xi) - b_t}{\mu} \rvu_{t - n, k}}^2] \\
    \overset{(a)}{=}& \frac{1}{N^2 K^2 \mu^2} \sum_{n, k = 1}^{N, K} \E\qty[\norm{\qty(f(\vtheta_{t - n} + \mu \rvu_{t - n, k}; \xi) - b_t) \rvu_{t - n, k}}^2] \\
    \overset{(b)}{=}& \frac{1}{N^2 K^2 \mu^2} \sum_{n, k = 1}^{N, K} \E\qty[\qty|f(\vtheta_{t - n} + \mu \rvu_{t - n, k}; \xi) - b_t|^2] \\
    \overset{(c)}{\leq}& \frac{1}{N^2 K^2 \mu^2} \sum_{n, k = 1}^{N, K} \E_{\rvu}\qty[\sigma_\xi^2 + \qty|F(\vtheta_{t - n} + \mu \rvu_{t - n, k}) - b_t|^2] \\
    \overset{(d)}{\leq}& \frac{1}{N^2 K^2 \mu^2} \qty(N K \sigma_\xi^2 + \sum_{n, k = 1}^{N, K} \E_{\rvu}\qty[\qty|F(\vtheta_{t - n} + \mu \rvu_{t - n, k}) - F_\mu(\vtheta_{t - n})|^2] + K \sum_{n = 1}^N \qty|F_\mu(\vtheta_{t - n}) - b_t|^2) \\
    \overset{(e)}{\leq}& \frac{1}{N K \mu^2} \qty(\sigma_\xi^2 + \sigma_\mu^2 + \frac{L_0^2}{N^2} \sum_{n = 1}^N \sum_{n' = 1}^N \norm{\vtheta_{t - n} - \vtheta_{t - n'}}^2) \ ,
\end{aligned}
\end{equation}
where $(a)$ is derived from the independence of $\rvu_{t - n, k}$ across different iterations $t$ and queries $k$, $(b)$ comes from the fact that $\rvu \sim \mathrm{Unif}(\sS^{d-1})$ and hence $\norm{\rvu}^2 = 1$, $(c)$ is obtained from Assump. \ref{assump:2}, and $(d)$ results from $\E_{\rvu}\qty[F(\vtheta_{t - n} + \mu \rvu_{t - n, k}) - F_\mu(\vtheta_{t - n})] = 0$. In step $(e)$, we apply Assump. \ref{assump:2} and the following inequality:
\begin{equation}
\begin{aligned}
    \qty|F_\mu(\vtheta_{t - n}) - b_t|^2 =& \qty|\frac{1}{N} \sum_{n' = 1}^N \qty(F_\mu(\vtheta_{t - n}) - F_\mu(\vtheta_{t - n'}))|^2 \overset{(a)}{\leq} \frac{1}{N} \sum_{n' = 1}^N \qty|F_\mu(\vtheta_{t - n}) - F_\mu(\vtheta_{t - n'})|^2 \\
    \overset{(b)}{\leq}& \frac{L_0^2}{N} \sum_{n' = 1}^N \norm{\vtheta_{t - n} - \vtheta_{t - n'}}^2 \ ,
\end{aligned}
\end{equation}
where $(a)$ follows from the Jensen's inequality and $(b)$ is derived from Assump. \ref{assump:1}.

Inserting the results of \ref{eq:f_var_2} into \ref{eq:f_var_1}, we have
\begin{equation}
\begin{aligned}
    \Var\qty(\hat{\nabla} F(\vtheta_{t-1})) \leq& \frac{1}{N K \mu^2} \qty(\sigma_\xi^2 + \sigma_\mu^2 + \frac{L_0^2}{N^2} \sum_{n = 1}^N \sum_{n' = 1}^N \norm{\vtheta_{t - n} - \vtheta_{t - n'}}^2) - \norm{\frac{1}{N} \sum_{n = 1}^N \nabla F_\mu \qty(\vtheta_{t - n})}^2 \ .
\end{aligned}
\end{equation}

Finally, the MSE of the gradient estimator $\hat{\nabla} F(\vtheta_{t-1})$ with respect to $\nabla F_\mu(\vtheta_{t-1})$ can be bounded as below:
\begin{equation}\label{eq:f_var_3}
\begin{aligned}
    &\E\qty[\norm{\hat{\nabla} F(\vtheta_{t-1}) - \nabla F_\mu(\vtheta_{t-1})}^2] \\
    \overset{(a)}{\leq}& \Var\qty(\hat{\nabla} F(\vtheta_{t-1})) + \norm{\frac{1}{N} \sum_{n = 1}^N \nabla F_\mu \qty(\vtheta_{t - n}) - \nabla F_\mu(\vtheta_{t-1})}^2 \\
    \overset{(b)}{\leq}& \frac{1}{N K \mu^2} \qty(\sigma_\xi^2 + \sigma_\mu^2 + \frac{L_0^2}{N^2} \sum_{n = 1}^N \sum_{n' = 1}^N \norm{\vtheta_{t - n} - \vtheta_{t - n'}}^2) + \frac{L_1^2 d}{N} \sum_{n = 1}^N \norm{\vtheta_{t - n} - \vtheta_{t-1}}^2 - \norm{\frac{1}{N} \sum_{n = 1}^N \nabla F_\mu \qty(\vtheta_{t - n})}^2 \ ,
\end{aligned}
\end{equation}
where $(a)$ follows from the fact that $\Var\qty(\hat{\nabla} F(\vtheta_{t-1}))$ and $\norm{\frac{1}{N} \sum_{n = 1}^N \nabla F_\mu \qty(\vtheta_{t - n}) - \nabla F_\mu(\vtheta_{t-1})}^2$ with respect to $\{\xi_\tau\}_\tau^t$ are independent, while $(b)$ is derived from the subsequent inequality using Jensen's inequality and Assump. \ref{assump:1}:
\begin{equation}
\begin{aligned}
    \norm{\frac{1}{N} \sum_{n = 1}^N \nabla F_\mu \qty(\vtheta_{t - n}) - \nabla F_\mu(\vtheta_{t-1})}^2 \leq& \frac{1}{N} \sum_{n = 1}^N \norm{\nabla F_\mu \qty(\vtheta_{t - n}) - \nabla F_\mu(\vtheta_{t-1})}^2 \leq \frac{L_1^2 d}{N} \sum_{n = 1}^N \norm{\vtheta_{t - n} - \vtheta_{t-1}}^2 \ .
\end{aligned}
\end{equation}

Considering the update rule of $\gR$-AdaZO, \eqref{eq:f_var_3} can be further simplified as:
\begin{equation}\label{eq:f_var_4}
\begin{aligned}
    \E\qty[\norm{\hat{\nabla} F(\vtheta_{t-1}) - \nabla F_\mu(\vtheta_{t-1})}^2] \leq& \frac{\sigma_\xi^2 + \sigma_\mu^2}{N K \mu^2} + \frac{L_0^2 \eta^2 d \qty(N^2 - 1)}{3 \qty(1 - \beta_2) N^2 K \mu^2} + \frac{L_1^2 \eta^2 d^2 \qty(N - 1)}{2 \qty(1 - \beta_2)} \ ,
\end{aligned}
\end{equation}
where we perform $\frac{\vm_{t, i}}{\sqrt{\vv_{t, i}}} \leq \frac{1}{\sqrt{1 - \beta_2}}$ in the following inequality:
\begin{equation}
\begin{aligned}
    \norm{\vtheta_{t - n} - \vtheta_{t-n'}}^2 \leq& \sum_{n'' = \min(n, n')}^{\max(n, n') - 1} \norm{\frac{\eta \vm_{t - n''}}{\sqrt{\vv_{t - n''} + \zeta}}}^2 \leq \sum_{n'' = \min(n, n')}^{\max(n, n') - 1} \frac{\eta^2 d}{1 - \beta_2} = \frac{\eta^2 d}{1 - \beta_2} \qty|n - n'| \ .
\end{aligned}
\end{equation}
\end{proof}

\subsection{Proof of Thm.~\ref{thm:conv-zoo}}\label{appx:proof-of-conv}
To ease the proof of Thm. \ref{thm:conv-zoo}, we first prove the smoothness of $F_\mu$ (Lemma \ref{lem:Fmu_smoothness}), then the upper bound of the squared first moment $\vm_{t, i}^2$ (Lemma \ref{lem:m_var}) and the second moment $\vv_{t, i}$ (Lemma \ref{lem:v_avg}).



\begin{lemma}\label{lem:Fmu_smoothness}
\textit{\fontfamily{ppl}\selectfont
$\forall \vtheta, \vtheta' \in \sR^d$, we have
\begin{equation*}
\begin{aligned}
    \qty|\nabla_i F_\mu(\vtheta) - \nabla_i F_\mu(\vtheta')| \leq& L_1 \norm{\vtheta - \vtheta'} \ . 
\end{aligned}
\end{equation*}
}
\end{lemma}
\begin{proof}
\begin{equation}
\begin{aligned}
    \qty|\nabla_i F_\mu(\vtheta) - \nabla_i F_\mu(\vtheta')| =& \qty|\E_{\rvu}\qty[\nabla_i F(\vtheta + \mu \rvu) - \nabla_i F(\vtheta' + \mu \rvu)]| \\
    \overset{(a)}{\leq}& \E_{\rvu}\qty[\qty|\nabla_i F(\vtheta + \mu \rvu) - \nabla_i F(\vtheta' + \mu \rvu)|] \\
    \overset{(b)}{\leq}& L_1 \norm{\vtheta - \vtheta'} \ ,
\end{aligned}
\end{equation}
where $(a)$ comes from the Jensen's inequality and $(b)$ follows from Assump. \ref{assump:2}.
\end{proof}

\begin{lemma}\label{lem:m_var}
\textit{\fontfamily{ppl}\selectfont
$\forall \vtheta \in \sR^d$, $i \in [d]$ and $t \in [T]$, if $\vm_{0, i} = 0$, the following inequality holds for \ours{}, 
\begin{equation*}
\begin{aligned}
    \vm_{t, i}^2 \leq& 2 \qty(1 + \beta_1) \qty(1 - \beta_1)^2 \sum_{\tau = 1}^t \beta_1^{2(t - \tau)} \qty|\hat{\nabla}_i f(\vtheta_{\tau - 1}; \xi_\tau) - \frac{1}{\min \qty(N, \tau)} \sum_{n = 1}^{\min \qty(N, \tau)} \nabla_i F_\mu(\vtheta_{\tau - n})|^2 \\
    &\quad + \frac{2 \beta_1 (1 + \beta_1)^2}{\qty(1 - \beta_1)^2 \qty(1 - \beta_2)} \eta^2 L_1^2 d C_N + 2 \frac{(1 + \beta_1)^2}{\beta_1} \qty|\nabla_i F_\mu(\vtheta_{t - 1})|^2 \ ,
\end{aligned}
\end{equation*}
where 
$C_N \triangleq \frac{2 \qty(1 - \beta_1)^2 N^2 - 3 \qty(1 - \beta_1) \qty(1 - 3 \beta_1) N - \beta_1 \qty(2 - 13 \beta_1) + 1}{6 \beta_1 (1 + \beta_1)}$ is monotonously increasing in $N$ and satisfies $C_N = 1$ when $N = 1$.}
\end{lemma}
\begin{proof}
First of all, the square of the first moment $\vm_{t, i}^2$ can be bounded as below:
\begin{equation}\label{eq:m_var_0}
\begin{aligned}
    \vm_{t, i}^2 =& |\vm_{t, i} - \nabla_i F_\mu(\vtheta_{t - 1}) + \nabla_i F_\mu(\vtheta_{t - 1})|^2 \\
    \leq& \qty(1 + \beta_1) |\vm_{t, i} - \nabla_i F_\mu(\vtheta_{t - 1})|^2 + \qty(1 + \frac{1}{\beta_1}) |\nabla_i F_\mu(\vtheta_{t - 1})|^2 \ ,
\end{aligned}
\end{equation}
where we apply the inequality $\qty(a + b)^2 \leq \qty(1 + \beta_1) a^2 + \qty(1 + \frac{1}{\beta_1}) b^2$.

The first term of \eqref{eq:m_var_0} can be further bounded:
\begin{equation}\label{eq:m_var_1}
\begin{aligned}
    &|\vm_{t, i} - \nabla_i F_\mu(\vtheta_{t - 1})|^2 \\ 
    =& |\vm_{t, i} - \E\qty[\vm_{t, i}] + \E\qty[\vm_{t, i}] - \nabla_i F_\mu(\vtheta_{t - 1})|^2 \\
    \overset{(a)}{\leq}& 2 |\vm_{t, i} - \E\qty[\vm_{t, i}]|^2 + 2 |\E\qty[\vm_{t, i}] - \nabla_i F_\mu(\vtheta_{t - 1})|^2 \\[5pt]
    =& 2 |\vm_{t, i} - \E\qty[\vm_{t, i}]|^2 + 2 |\E\qty[\vm_{t, i}] - \qty(1 - \beta_1^t) \nabla_i F_\mu(\vtheta_{t - 1}) - \beta_1^t \nabla_i F_\mu(\vtheta_{t - 1})|^2 \\
    \overset{(b)}{\leq}& 2 |\vm_{t, i} - \E\qty[\vm_{t, i}]|^2 + 2 \qty(1 - \beta_1^t) \qty| \frac{\E\qty[\vm_{t, i}]}{1 - \beta_1^t}  - \nabla_i F_\mu(\vtheta_{t - 1})|^2 + 2 \beta_1^t \qty|\nabla_i F_\mu(\vtheta_{t - 1})|^2  \ ,
\end{aligned}
\end{equation}
where $(a)$ and $(b)$ utilize the inequality $\qty(a + b)^2 \leq \qty(1 + k) a^2 + \qty(1 + \frac{1}{k}) b^2$ for any $k > 0$, with $k = 1$ in step $(a)$ and $k = \frac{1 - \beta_1^t}{\beta_1^t}$ in step $(b)$.

We next bound the first and second terms of \eqref{eq:m_var_1} separately. First, by assuming $\vm_{0, i} = 0$, the geometric series of $\vm_{t, i}$ and $\E\qty[\vm_{t, i}]$ are given by:
\begin{equation}\label{eq:m_series}
\begin{aligned}
    \vm_{t, i} =& (1 - \beta_1) \sum_{\tau = 1}^t \beta_1^{t - \tau} \hat{\nabla}_i f(\vtheta_{\tau - 1}; \xi_\tau); \\
    \E\qty[\vm_{t, i}] =& (1 - \beta_1) \sum_{\tau = 1}^t \beta_1^{t - \tau} \frac{1}{\min \qty(N, \tau)} \sum_{n = 1}^{\min \qty(N, \tau)} \nabla_i F_\mu(\vtheta_{\tau - n}) \ .
\end{aligned}
\end{equation}

Therefore, the first term of \eqref{eq:m_var_1} can be bounded as:
\begin{equation}\label{eq:m_var_2}
\begin{aligned}
    |\vm_{t, i} - \E\qty[\vm_{t, i}]|^2 =& \qty| (1 - \beta_1) \sum_{\tau = 1}^t \beta_1^{t - \tau} \qty(\hat{\nabla}_i f(\vtheta_{\tau - 1}; \xi_\tau) - \frac{1}{\min \qty(N, \tau)} \sum_{n = 1}^{\min \qty(N, \tau)} \nabla_i F_\mu(\vtheta_{\tau - n}))|^2 \\
    \overset{(a)}{=}& (1 - \beta_1)^2 \sum_{\tau = 1}^t \beta_1^{2(t - \tau)} \qty|\hat{\nabla}_i f(\vtheta_{\tau - 1}; \xi_\tau) - \frac{1}{\min \qty(N, \tau)} \sum_{n = 1}^{\min \qty(N, \tau)} \nabla_i F_\mu(\vtheta_{\tau - n})|^2 \ ,
\end{aligned}
\end{equation}
where $(a)$ results from the independence of different $\{\xi_\tau\}_{\tau = 1}^t$.



Besides, the second term of \eqref{eq:m_var_1} can be bounded as below:
\begin{equation}\label{eq:m_var_3}
\begin{aligned}
    &\qty| \frac{\E\qty[\vm_{t, i}]}{1 - \beta_1^t}  - \nabla_i F_\mu(\vtheta_{t - 1})|^2 \\ 
    =& \qty| \frac{(1 - \beta_1)}{1 - \beta_1^t} \sum_{\tau = 1}^t \beta_1^{t - \tau} \qty(\frac{1}{\min \qty(N, \tau)} \sum_{n = 1}^{\min \qty(N, \tau)} \nabla_i F_\mu(\vtheta_{\tau - n})  - \nabla_i F_\mu(\vtheta_{t - 1}))|^2 \\
    =& \frac{(1 - \beta_1)^2}{\qty(1 - \beta_1^t)^2} \sum_{\tau, \tau' = 1}^t \beta_1^{2t - \tau - \tau'} \qty(\frac{1}{\min \qty(N, \tau)} \sum_{n = 1}^{\min \qty(N, \tau)} \nabla_i F_\mu(\vtheta_{\tau - n}) - \nabla_i F_\mu(\vtheta_{t - 1})) \\
    &\hspace{11em} \times \qty(\frac{1}{\min \qty(N, \tau)} \sum_{n = 1}^{\min \qty(N, \tau)} \nabla_i F_\mu(\vtheta_{\tau' - n})  - \nabla_i F_\mu(\vtheta_{t - 1})) \\
    \overset{(a)}{\leq}& \frac{(1 - \beta_1)^2}{2 \qty(1 - \beta_1^t)^2} \sum_{\tau, \tau' = 1}^t \beta_1^{2t - \tau - \tau'} \left(\qty|\frac{1}{\min \qty(N, \tau)} \sum_{n = 1}^{\min \qty(N, \tau)} \nabla_i F_\mu(\vtheta_{\tau - n})  - \nabla_i F_\mu(\vtheta_{t - 1})|^2\right. \\
    &\hspace{11.5em} \left.+ \qty|\frac{1}{\min \qty(N, \tau)} \sum_{n = 1}^{\min \qty(N, \tau)} \nabla_i F_\mu(\vtheta_{\tau - n})  - \nabla_i F_\mu(\vtheta_{t - 1})|^2\right) \\
    =& \frac{1 - \beta_1}{1 - \beta_1^t} \sum_{\tau = 1}^t \beta_1^{t - \tau} \qty|\frac{1}{\min \qty(N, \tau)} \sum_{n = 1}^{\min \qty(N, \tau)} \nabla_i F_\mu(\vtheta_{\tau - n})  - \nabla_i F_\mu(\vtheta_{t - 1})|^2 \\
    \overset{(b)}{\leq}& \frac{1 - \beta_1}{1 - \beta_1^t} \sum_{\tau = 1}^t \beta_1^{t - \tau} \frac{1}{\min \qty(N, \tau)} \sum_{n = 1}^{\min \qty(N, \tau)} \qty|\nabla_i F_\mu(\vtheta_{\tau - n})  - \nabla_i F_\mu(\vtheta_{t - 1})|^2 \\
    \overset{(c)}{\leq}& \frac{1 - \beta_1}{1 - \beta_1^t} L_1^2 \sum_{\tau = 1}^t \beta_1^{t - \tau} \frac{1}{\min \qty(N, \tau)} \sum_{n = 1}^{\min \qty(N, \tau)} \|\vtheta_{\tau - n} - \vtheta_{t - 1}\|^2 \ ,
\end{aligned}
\end{equation}
where $(a)$ is derived from $ab \leq \half (a^2 + b^2)$, $(b)$ is due to Jensen's inequality, and $(c)$ is obtained from Lemma \ref{lem:Fmu_smoothness}.

Recalled to the update rule of $\gR$-AdaZO, we have:
\begin{equation}\label{eq:m_var_4}
\begin{aligned}
    \|\vtheta_{\tau - n} - \vtheta_{t - 1}\|^2 =& \sum_i^d |\vtheta_{\tau - n, i} - \vtheta_{t - 1, i}|^2 = \eta^2 \sum_i^d \qty|\sum_{s = \tau - n + 1}^{t - 1} \frac{\vm_{s, i}}{\sqrt{\vv_{s, i} + \zeta}}|^2 \\
    \overset{(a)}{\leq}& \eta^2 \qty(t - \tau + n - 1) \sum_i^d \sum_{s = \tau - n + 1}^{t - 1} \frac{\vm_{s, i}^2}{\vv_{s, i} + \zeta} \overset{(b)}{\leq} \frac{d}{1 - \beta_2} \eta^2 \qty(t - \tau + n - 1)^2 \ ,
\end{aligned}
\end{equation}
where $(a)$ is from Cauchy-Schwarz inequality $\qty|\sum_{s = \tau - n + 1}^{t - 1} a_s|^2 \leq \qty(t - \tau + n - 1) \sum_{s = \tau - n + 1}^{t - 1} a_s^2$, and $(b)$ follows from $\frac{\vm_{s, i}^2}{\vv_{s, i} + \zeta} \leq \frac{1}{1 - \beta_2}$. 

Putting the result of \eqref{eq:m_var_4} into \eqref{eq:m_var_3}, we have:
\begin{equation}\label{eq:m_var_5}
\begin{aligned}
    \qty| \frac{\E\qty[\vm_{t, i}]}{1 - \beta_1^t}  - \nabla_i F_\mu(\vtheta_{t - 1})|^2 \leq& \frac{\qty(1 - \beta_1)  \eta^2 L_1^2 d}{\qty(1 - \beta_1^t) \qty(1 - \beta_2)} \sum_{\tau = 1}^t \frac{\beta_1^{t - \tau}}{\min \qty(N, \tau)} \sum_{n = 1}^{\min \qty(N, \tau)} \qty(t - \tau + n - 1)^2 \\
    \overset{(a)}{\leq}& \frac{\beta_1 (1 + \beta_1)}{\qty(1 - \beta_1^t) \qty(1 - \beta_1)^2 \qty(1 - \beta_2)} \eta^2 L_1^2 d C_N \ ,
\end{aligned}
\end{equation}
where $(a)$ comes from the geometric series summation over $\tau$ and $n$:
\begin{equation}\label{eq:m_var_6}
    \sum_{\tau = 1}^t \frac{1}{\min \qty(N, \tau)} \sum_{n = 1}^{\min \qty(N, \tau)} \beta_1^{t - \tau} \qty(t - \tau + n - 1)^2 \leq \frac{\beta_1 (1 + \beta_1)}{\qty(1 - \beta_1)^3} C_N \ ,
\end{equation}
where $C_N \triangleq \frac{2 \qty(1 - \beta_1)^2 N^2 - 3 \qty(1 - \beta_1) \qty(1 - 3 \beta_1) N - \beta_1 \qty(2 - 13 \beta_1) + 1}{6 \beta_1 (1 + \beta_1)}$ is monotonously increasing in $N$ and satisfies $C_N = 1$ when $N = 1$.

Finally, gathering the results of \eqref{eq:m_var_2}, \eqref{eq:m_var_5} into \eqref{eq:m_var_1}, we obtain:
\begin{equation}
\begin{aligned}
    |\vm_{t, i} - \nabla_i F_\mu(\vtheta_{t - 1})|^2 \leq& (1 - \beta_1)^2 \sum_{\tau = 1}^t \beta_1^{2(t - \tau)} \qty|\hat{\nabla}_i f(\vtheta_{\tau - 1}; \xi_\tau) - \frac{1}{\min \qty(N, \tau)} \sum_{n = 1}^{\min \qty(N, \tau)} \nabla_i F_\mu(\vtheta_{\tau - n})|^2 \\
    &\quad + \frac{\beta_1 (1 + \beta_1)}{\qty(1 - \beta_1)^2 \qty(1 - \beta_2)} \eta^2 L_1^2 d C_N + \beta_1^t \qty|\nabla_i F_\mu(\vtheta_{t - 1})|^2 \\
    &\quad \quad + 2 \qty(m_{\tau, i} - \E\qty[\vm_{\tau, i}]) \qty(\E\qty[\vm_{\tau, i}] - \nabla_i F_\mu(\theta_{\tau - 1})) \ .
\end{aligned}
\end{equation}

Consequently, \eqref{eq:m_var_0} becomes:
\begin{equation}
\begin{aligned}
    \vm_{t, i}^2 \leq& 2 \qty(1 + \beta_1) (1 - \beta_1)^2 \sum_{\tau = 1}^t \beta_1^{2(t - \tau)} \qty|\hat{\nabla}_i f(\vtheta_{\tau - 1}; \xi_\tau) - \frac{1}{\min \qty(N, \tau)} \sum_{n = 1}^{\min \qty(N, \tau)} \nabla_i F_\mu(\vtheta_{\tau - n})|^2 \\
    & \quad + \frac{2 \beta_1 (1 + \beta_1)^2}{\qty(1 - \beta_1)^2 \qty(1 - \beta_2)} \eta^2 L_1^2 d C_N + 2 \qty(1 + \beta_1) \beta_1^t \qty|\nabla_i F_\mu(\vtheta_{t - 1})|^2 + 2 \qty(1 + \frac{1}{\beta_1}) \qty|\nabla_i F_\mu(\vtheta_{t - 1})|^2 \\
    \leq& 2 \qty(1 + \beta_1) \qty(1 - \beta_1)^2 \sum_{\tau = 1}^t \beta_1^{2(t - \tau)} \qty|\hat{\nabla}_i f(\vtheta_{\tau - 1}; \xi_\tau) - \frac{1}{\min \qty(N, \tau)} \sum_{n = 1}^{\min \qty(N, \tau)} \nabla_i F_\mu(\vtheta_{\tau - n})|^2 \\
    &\quad + \frac{2 \beta_1 (1 + \beta_1)^2}{\qty(1 - \beta_1)^2 \qty(1 - \beta_2)} \eta^2 L_1^2 d C_N + \frac{2 (1 + \beta_1)^2}{\beta_1} \qty|\nabla_i F_\mu(\vtheta_{t - 1})|^2 \ ,
\end{aligned}
\end{equation}
which concludes the proof.
\end{proof}

\begin{lemma}
\label{lem:v_avg}
\textit{\fontfamily{ppl}\selectfont
$\forall \vtheta \in \sR^d$, $i \in [d]$ and $t \in [T]$, if $\vv_{0, i} > 0$, the following inequality holds for \ours{},
\begin{equation*}
\begin{aligned}
    \vv_{t, i} \leq& \beta_2^t \vv_{0, i} + \frac{2 \qty(1 + \beta_1) \qty(1 - \beta_1)^2 \qty(1 - \beta_2)}{\beta_2 - \beta_1^2} \sum_{\tau = 1}^t \qty(\beta_2^{t + 1 - \tau} - \beta_1^{2 \qty(t + 1 - \tau)}) \qty|\hat{\nabla}_i f(\vtheta_{\tau - 1}; \xi_\tau) - \frac{1}{\min \qty(N, \tau)} \sum_{n = 1}^{\min \qty(N, \tau)} \nabla_i F_\mu(\vtheta_{\tau - n})|^2 \\
    &\quad + \frac{2 \beta_1 (1 + \beta_1)^2}{\qty(1 - \beta_1)^2 \qty(1 - \beta_2)} L_1^2 \eta^2 d C_N + \frac{2 \qty(1 + \beta_1)^2 \qty(1 - \beta_2)}{\beta_1} \sum_{\tau = 1}^t \beta_2^{t - \tau}  |\nabla_i F_\mu(\vtheta_{\tau - 1})|^2 \ ,
\end{aligned}
\end{equation*}
where 
$C_N \triangleq \frac{2 \qty(1 - \beta_1)^2 N^2 - 3 \qty(1 - \beta_1) \qty(1 - 3 \beta_1) N - \beta_1 \qty(2 - 13 \beta_1) + 1}{6 \beta_1 (1 + \beta_1)}$ is monotonously increasing in $N$ and satisfies $C_N = 1$ when $N = 1$.
}
\end{lemma}
\begin{proof}
Starting the geometric series of $\vv_{t, i}$, the following inequality holds:
\begin{equation}\label{eq:v_exp}
\begin{aligned}
    &\vv_{t, i} = \beta_2^t \vv_{0, i} + \qty(1 - \beta_2) \sum_{\tau = 1}^t \beta_2^{t - \tau} \vm_{\tau, i}^2 \\[4pt]
    \overset{(a)}{\leq}& \beta_2^t \vv_{0, i} + \qty(1 - \beta_2) \sum_{\tau = 1}^t \beta_2^{t - \tau} \Bigg(2 \qty(1 + \beta_1) \qty(1 - \beta_1)^2 \sum_{s = 1}^\tau \beta_1^{2(\tau - s)} \qty|\hat{\nabla}_i f(\vtheta_{s - 1}; \xi_s) - \frac{1}{\min \qty(N, s)} \sum_{n = 1}^{\min \qty(N, s)} \nabla_i F_\mu(\vtheta_{s - n})|^2 \\
    &\quad + \frac{2 \beta_1 (1 + \beta_1)^2}{\qty(1 - \beta_1)^2 \qty(1 - \beta_2)} L_1^2 \eta^2 d C_N + \frac{2 \qty(1 + \beta_1)^2}{\beta_1} |\nabla_i F_\mu(\vtheta_{\tau - 1})|^2 \Bigg) \\[7pt]
    \overset{(b)}{=}& \beta_2^t \vv_{0, i} + \frac{2 \qty(1 + \beta_1) \qty(1 - \beta_1)^2 \qty(1 - \beta_2)}{\beta_2 - \beta_1^2} \sum_{\tau = 1}^t \qty(\beta_2^{t + 1 - \tau} - \beta_1^{2 \qty(t + 1 - \tau)}) \qty|\hat{\nabla}_i f(\vtheta_{\tau - 1}; \xi_\tau) - \frac{1}{\min \qty(N, \tau)} \sum_{n = 1}^{\min \qty(N, \tau)} \nabla_i F_\mu(\vtheta_{\tau - n})|^2 \\
    &\quad + \frac{2 \beta_1 (1 + \beta_1)^2 \qty(1 - \beta_2^t)}{\qty(1 - \beta_1)^2 \qty(1 - \beta_2)} L_1^2 \eta^2 d C_N + \frac{2 \qty(1 + \beta_1)^2 \qty(1 - \beta_2)}{\beta_1} \sum_{\tau = 1}^t \beta_2^{t - \tau} |\nabla_i F_\mu(\vtheta_{\tau - 1})|^2 \\
    \overset{(c)}{\leq}& \beta_2^t \vv_{0, i} + \frac{2 \qty(1 + \beta_1) \qty(1 - \beta_1)^2 \qty(1 - \beta_2)}{\beta_2 - \beta_1^2} \sum_{\tau = 1}^t \qty(\beta_2^{t + 1 - \tau} - \beta_1^{2 \qty(t + 1 - \tau)}) \qty|\hat{\nabla}_i f(\vtheta_{\tau - 1}; \xi_\tau) - \frac{1}{\min \qty(N, \tau)} \sum_{n = 1}^{\min \qty(N, \tau)} \nabla_i F_\mu(\vtheta_{\tau - n})|^2 \\
    &\quad + \frac{2 \beta_1 (1 + \beta_1)^2}{\qty(1 - \beta_1)^2 \qty(1 - \beta_2)} L_1^2 \eta^2 d C_N + \frac{2 \qty(1 + \beta_1)^2 \qty(1 - \beta_2)}{\beta_1} \sum_{\tau = 1}^t \beta_2^{t - \tau}  |\nabla_i F_\mu(\vtheta_{\tau - 1})|^2 \ ,
\end{aligned}
\end{equation}
where $(a)$ comes from Lemma \ref{lem:m_var}, and $(c)$ is due to $\beta_2 \leq 1$. In step $(b)$, we use the following geometric series summation:
\begin{equation}
\begin{aligned}
    & \sum_{\tau = 1}^t \sum_{s = 1}^\tau \beta_2^{t - \tau} \beta_1^{2(\tau - s)} \qty|\hat{\nabla}_i f(\vtheta_{s - 1}; \xi_s) - \frac{1}{\min \qty(N, s)} \sum_{n = 1}^{\min \qty(N, s)} \nabla_i F_\mu(\vtheta_{s - n})|^2 \\
    \overset{(a)}{=}& \sum_{s = 1}^t \sum_{\tau = s}^t \beta_2^{t - \tau} \beta_1^{2(\tau - s)} \qty|\hat{\nabla}_i f(\vtheta_{s - 1}; \xi_s) - \frac{1}{\min \qty(N, s)} \sum_{n = 1}^{\min \qty(N, s)} \nabla_i F_\mu(\vtheta_{s - n})|^2 \\
    =& \frac{1}{\beta_2 - \beta_1^2} \sum_{\tau = 1}^t \qty(\beta_2^{t + 1 - \tau} - \beta_1^{2 \qty(t + 1 - \tau)}) \qty|\hat{\nabla}_i f(\vtheta_{\tau - 1}; \xi_\tau) - \frac{1}{\min \qty(N, \tau)} \sum_{n = 1}^{\min \qty(N, \tau)} \nabla_i F_\mu(\vtheta_{\tau - n})|^2 \ ,
\end{aligned}
\end{equation}
where we exchange the order of summation in step $(a)$. This concludes the proof.
\end{proof}


Here, we give the formal statement of the convergence of \ours{} in Thm. \ref{thm:conv-zoo}.
\begin{theorem}[Variance-Aware Convergence of $\gR$-AdaZO, Formal]\label{thm:conv-zoo-formal}
\textit{\fontfamily{ppl}\selectfont
Let $\rvu \sim \mathrm{Unif}(\sS^{d-1})$ and $b_t$ in \eqref{eq:grad-avg} be the optimal $b_t^*$ in Thm.~\ref{thm:opt-baseline}, under the Assump.~\ref{assump:1}, \ref{assump:2}, and the gradient of $F_\mu$ being bounded $\qty|\nabla_i F_\mu(\vtheta_t)| \leq G_\mu$, when 
$\eta = \frac{\qty(1 - \beta_1) \qty(1 - \beta_1/\sqrt{\beta_2}) \epsilon^2}{128 L_1 d^{3 / 2}} \sim \gO\qty(\epsilon^2)$, $1 - \beta_2 = \min \qty(\frac{\qty(1 - \beta_1)^2 \mu^2 \eta \sqrt{\zeta} \epsilon^2}{64 C^2 d^3}, \frac{\qty(1 - \beta_1)^2 \epsilon^2}{4 \beta_1^2 d \sqrt{\zeta}}) \sim \gO\qty(\epsilon^2)$, $T = \max \qty(\frac{64 C^2 d^3}{\qty(1 - \beta_1)^2 \mu^2 \eta \sqrt{\zeta} \eta \epsilon^2}, \frac{8 \qty(1 - \beta_1/\sqrt{\beta_2})}{\qty(1 - \beta_1) \eta \epsilon^2}, \frac{64 L_1 \sqrt{d} \eta}{\qty(1 - \beta_1) \qty(1 - \beta_1/\sqrt{\beta_2}) \qty(1 - \beta_2) \epsilon^2} \sum_i^d \ln \qty(\frac{\beta_2^T \vv_{0, i} + 4 C^2 d^2 / \mu^2}{\vv_{0, i}})) \sim \gO\qty(\epsilon^{-4})$,
$\beta_1 \leq \sqrt{\beta_2}$, $\beta_2 > 1/2$, $\vm_{0, i} = 0$, $\vv_{0, i} > 0$ ($\forall i \in [d]$), the following convergence holds for \ours{} (Algo.~\ref{alg:ours}),
    \begin{equation} 
        \frac1T \sum_{t=1}^T \E[\|\nabla F(\vtheta_t)\|] \leq \sqrt{\frac{2}{\beta_1 \qty(1 - \beta_2)}} \qty(1 + \beta_1) \epsilon^2 + \qty(\sqrt[4]{\zeta} + \sqrt{\Xi}) \epsilon + \mu L_1 \sqrt{d} + B_2 \ ,
    \end{equation}
    where 
    $B_1 \triangleq \sqrt{d \beta_2 \norm{\vv_{0}} + \frac{2 \beta_1 (1 + \beta_1)^2}{\qty(1 - \beta_1)^2 \qty(1 - \beta_2)} L_1^2 \eta^2 d^2 C_N} + \sqrt{\frac{2 \qty(1 + \beta_1) \qty(1 - \beta_1)^2 \beta_2 L_0^2 \eta^2 d \qty(N^2 - 1)}{3 \qty(\beta_2 - \beta_1^2) \qty(1 - \beta_2)^2 N^2 K \mu^2}}$, 
    $\Xi \triangleq B_1 + \sqrt{\frac{2 \qty(1 + \beta_1) \qty(1 - \beta_1)^2 \beta_2}{\qty(\beta_2 - \beta_1^2) \qty(1 - \beta_2)} V}$, 
    $V = \frac{\sigma_\xi^2 + \sigma_\mu^2}{N K \mu^2}$, 
    $G \triangleq \frac{2 G_\mu}{\sqrt{\zeta}} \sqrt{d \qty(V + \frac{L_0^2 \eta^2 d \qty(N^2 - 1)}{3 \qty(1 - \beta_2) N^2 K \mu^2} + \frac{L_1^2 \eta^2 d^2 \qty(N - 1)}{2 \qty(1 - \beta_2)})}$, 
    $B_2 \triangleq \sqrt{\frac{2}{\beta_1 \qty(1 - \beta_2)}} \qty(1 + \beta_1) G + \qty(\sqrt[4]{\zeta} + \sqrt{\Xi}) \sqrt{G}$, and 
    $C_N \triangleq \frac{2 \qty(1 - \beta_1)^2 N^2 - 3 \qty(1 - \beta_1) \qty(1 - 3 \beta_1) N - \beta_1 \qty(2 - 13 \beta_1) + 1}{6 \beta_1 (1 + \beta_1)}$ is monotonously increasing in $N$ and satisfies $C_N = 1$ when $N = 1$.
}
\end{theorem}
\begin{proof}
We begin by introducing the following transformation:
\begin{equation}
\begin{aligned}\label{eq:1}
    \qty(\frac1T \sum_{t=1}^T \E\qty[\|\nabla F_\mu(\vtheta_t)\|])^2 =& \qty(\frac1T \sum_{t=1}^T \E\qty[\sqrt[4]{\beta_2 \|\vv_t\| + \zeta} \cdot \frac{\|\nabla F_\mu(\vtheta_t)\|}{\sqrt[4]{\beta_2 \|\vv_t\| + \zeta}}])^2 \\
    \overset{(a)}{\leq}& \frac{1}{T^2} \qty(\sum_{t=1}^T \E\qty[\sqrt{\beta_2 \|\vv_t\| + \zeta}]^\half \cdot \E\qty[\frac{\|\nabla F_\mu(\vtheta_t)\|^2}{\sqrt{\beta_2 \|\vv_t\| + \zeta}}]^\half)^2 \\
    \overset{(b)}{\leq}& \underbrace{\frac{1}{T} \sum_{t=1}^T \E\qty[\sqrt{\beta_2 \|\vv_t\| + \zeta}]}_{\text{Term I}} \cdot \underbrace{\frac{1}{T} \sum_{t=1}^T \E\qty[\frac{\|\nabla F_\mu(\vtheta_t)\|^2}{\sqrt{\beta_2 \|\vv_t\| + \zeta}}]}_{\text{Term II}} \ ,
\end{aligned}
\end{equation}
where $(a)$ comes from the Hölder's inequality $\E[|\va \vb|]\leq\qty(\E\qty[|\va|^2])^{\frac{1}{2}}\qty(\E\qty[|\vb|^2])^{\frac{1}{2}}$, and $(b)$ results from the Cauchy-Schwarz inequality.

\paragraph{Calculation of Term I.}
Based on Lemma \ref{lem:v_avg}, $\vv_{0, i} \leq \norm{\vv_0}$, and $\beta_2 \leq 1$, we have:
\begin{equation}
\begin{aligned}
    \vv_{t, i} \leq& \beta_2 \norm{\vv_0} + \frac{2 \qty(1 + \beta_1) \qty(1 - \beta_1)^2 \qty(1 - \beta_2)}{\beta_2 - \beta_1^2} \sum_{\tau = 1}^t \qty(\beta_2^{t + 1 - \tau} - \beta_1^{2 \qty(t + 1 - \tau)}) \qty|\hat{\nabla}_i f(\vtheta_{\tau - 1}; \xi_\tau) - \frac{1}{\min \qty(N, \tau)} \sum_{n = 1}^{\min \qty(N, \tau)} \nabla_i F_\mu(\vtheta_{\tau - n})|^2 \\
    &\quad + \frac{2 \beta_1 (1 + \beta_1)^2}{\qty(1 - \beta_1)^2 \qty(1 - \beta_2)} L_1^2 \eta^2 d C_N + \frac{2 \qty(1 + \beta_1)^2 \qty(1 - \beta_2)}{\beta_1} \sum_{\tau = 1}^t \beta_2^{t - \tau}  |\nabla_i F_\mu(\vtheta_{\tau - 1})|^2 \ .
\end{aligned}
\end{equation}

Therefore, the square root of the summed second moment can be bounded as follows:
\begin{equation}\label{eq:v_sum}
\begin{aligned}
    \sqrt{\sum_i^d \vv_{t, i}} \leq& \sqrt{d \beta_2 \norm{\vv_{0}} + \frac{2 \beta_1 (1 + \beta_1)^2}{\qty(1 - \beta_1)^2 \qty(1 - \beta_2)} L_1^2 \eta^2 d^2 C_N} + \sqrt{\frac{2 \qty(1 - \beta_2)}{\beta_1}} \qty(1 + \beta_1) \sum_{\tau = 1}^t \beta_2^{\frac{t - \tau}{2}} \norm{\nabla F_\mu (\theta_{\tau - 1})} \\
    &\quad + \sqrt{\frac{2 \qty(1 + \beta_1) \qty(1 - \beta_2)}{\beta_2 - \beta_1^2}} \qty(1 - \beta_1) \sum_{\tau = 1}^t \sqrt{\beta_2^{t + 1 - \tau} - \beta_1^{2 \qty(t + 1 - \tau)}} \norm{\hat{\nabla} f(\vtheta_{\tau - 1}; \xi_\tau) - \frac{1}{\min \qty(N, \tau)} \sum_{n = 1}^{\min \qty(N, \tau)} \nabla F_\mu(\vtheta_{\tau - n})} \ ,
\end{aligned}
\end{equation}
where we utilize the inequality $\sqrt{\sum_i a_i} \leq \sum_i \sqrt{a_i}$.

Subsequently, the expectation of \eqref{eq:v_sum} can be bounded as follows:
\begin{equation}
\begin{aligned}
    \E\qty[\sqrt{\sum_i^d \vv_{t, i}}] \leq& \sqrt{d \beta_2 \norm{\vv_{0}} + \frac{2 \beta_1 (1 + \beta_1)^2}{\qty(1 - \beta_1)^2 \qty(1 - \beta_2)} L_1^2 \eta^2 d^2 C_N} + \sqrt{\frac{2 \qty(1 - \beta_2)}{\beta_1}} \qty(1 + \beta_1) \sum_{\tau = 1}^t \beta_2^{\frac{t - \tau}{2}} \E\qty[\norm{\nabla_i F_\mu (\theta_{\tau - 1})}] \\
    &\quad + \sqrt{\frac{2 \qty(1 + \beta_1) \qty(1 - \beta_2)}{\beta_2 - \beta_1^2}} \qty(1 - \beta_1) \sum_{\tau = 1}^t \sqrt{\beta_2^{t + 1 - \tau} - \beta_1^{2 \qty(t + 1 - \tau)}} \sqrt{\Var\qty(\hat{\nabla} F(\vtheta_{\tau - 1}))} \ ,
\end{aligned}
\end{equation}
where we apply the following inequality by Jensen's inequality:
\begin{equation}
\begin{aligned}
    \E\qty[\norm{\hat{\nabla} f(\vtheta_{\tau - 1}; \xi_\tau) - \frac{1}{\min \qty(N, \tau)} \sum_{n = 1}^{\min \qty(N, \tau)} \nabla F_\mu(\vtheta_{\tau - n})}] \leq& \sqrt{\E\qty[\norm{\hat{\nabla} f(\vtheta_{\tau - 1}; \xi_\tau) - \frac{1}{\min \qty(N, \tau)} \sum_{n = 1}^{\min \qty(N, \tau)} \nabla F_\mu(\vtheta_{\tau - n})}^2]} \\
    =& \sqrt{\Var\qty(\hat{\nabla} F(\vtheta_{\tau - 1}))} \ ,
\end{aligned}
\end{equation}
where the definition of $\Var\qty(\hat{\nabla} F(\vtheta_{\tau - 1}))$ comes from Thm. \ref{thm:opt-baseline}.


Considering the average over all iterations $t$, we have:
\begin{equation}\label{eq:termI_0}
\begin{aligned}
    &\frac{1}{T} \sum_{t=1}^T \E\qty[\sqrt{\sum_i^d \vv_{t, i}}] \leq \sqrt{d \beta_2 \norm{\vv_{0}} + \frac{2 \beta_1 (1 + \beta_1)^2}{\qty(1 - \beta_1)^2 \qty(1 - \beta_2)} L_1^2 \eta^2 d^2 C_N} \\
    &\quad + \sqrt{\frac{2 \qty(1 + \beta_1) \qty(1 - \beta_2)}{\beta_2 - \beta_1^2}} \qty(1 - \beta_1) \frac{1}{T} \sum_{t=1}^T \sum_{\tau = 1}^t \sqrt{\beta_2^{t + 1 - \tau} - \beta_1^{2 \qty(t + 1 - \tau)}} \sqrt{\Var\qty(\hat{\nabla} F(\vtheta_{\tau - 1}))} \\
    &\quad \quad + \sqrt{\frac{2 \qty(1 - \beta_2)}{\beta_1}} \qty(1 + \beta_1) \frac{1}{T} \sum_{t=1}^T \sum_{\tau = 1}^t \beta_2^{\frac{t - \tau}{2}} \E\qty[\norm{\nabla F_\mu (\theta_{\tau - 1})}] \ .
\end{aligned}
\end{equation}

The second and third terms in \eqref{eq:termI_0} contain double geometric series summations over $\tau$ and $t$. For the second term in \eqref{eq:termI_0}, we have:
\begin{equation}
\begin{aligned}
    &\frac{1}{T} \sum_{t=1}^T \sum_{\tau = 1}^t \sqrt{\qty(\beta_2^{t + 1 - \tau} - \beta_1^{2 \qty(t + 1 - \tau)})} \sqrt{\Var\qty(\hat{\nabla} F(\vtheta_{\tau - 1}))} \\
    \overset{(a)}{\leq}& \frac{1}{T} \sum_{t=1}^T \sum_{\tau = 1}^t \qty(\beta_2^{\frac{t + 1 - \tau}{2}} - \beta_1^{t + 1 - \tau}) \sqrt{\Var\qty(\hat{\nabla} F(\vtheta_{\tau - 1}))} \\
    \overset{(b)}{\leq}& \frac{1}{T} \sum_{t=1}^T \sum_{\tau = 1}^t \beta_2^{\frac{t + 1 - \tau}{2}} \sqrt{\Var\qty(\hat{\nabla} F(\vtheta_{\tau - 1}))} \\
    \overset{(c)}{=}& \frac{1}{T} \sum_{\tau = 1}^T \sum_{t=\tau}^T \beta_2^{\frac{t + 1 - \tau}{2}} \sqrt{\Var\qty(\hat{\nabla} F(\vtheta_{\tau - 1}))} \\
    =& \frac{1}{T} \sum_{t = 1}^T \frac{\sqrt{\beta_2} - \sqrt{\beta_2^{2 + T - t}}}{1 - \sqrt{\beta_2}} \sqrt{\Var\qty(\hat{\nabla} F(\vtheta_{t - 1}))} \\
    \overset{(d)}{\leq}& \frac{\sqrt{\beta_2}}{1 - \beta_2} \frac{1}{T} \sum_{t = 1}^T \sqrt{\Var\qty(\hat{\nabla} F(\vtheta_{t - 1}))} \\
    \overset{(e)}{\leq}& \frac{\sqrt{\beta_2}}{1 - \beta_2} \sqrt{V + \frac{L_0^2 \eta^2 d \qty(N^2 - 1)}{3 \qty(1 - \beta_2) N^2 K \mu^2}} \\
    \overset{(f)}{\leq}& \frac{\sqrt{\beta_2}}{1 - \beta_2} \qty(\sqrt{V} + \sqrt{\frac{L_0^2 \eta^2 d \qty(N^2 - 1)}{3 \qty(1 - \beta_2) N^2 K \mu^2}}) \ ,
\end{aligned}
\end{equation}
where $(a)$, $(f)$ results from $\sqrt{\sum_i a_i} \leq \sum_i \sqrt{a_i}$, $(b)$, $(d)$ comes from $0 \leq \beta_1^2 \leq \beta_2 \leq 1$, and $(e)$ is due to \eqref{eq:f_var_4} in Appx. \ref{appx:proof-of-bias-variance-decomp} and $V \triangleq \frac{\sigma_\xi^2 + \sigma_\mu^2}{N K \mu^2}$. In step $(c)$ we exchange the order of summation over $t$ and $\tau$.

For the third term in \eqref{eq:termI_0}, we have:
\begin{equation}
\begin{aligned}
    \frac{1}{T} \sum_{t = 1}^T \sum_{\tau = 1}^t \beta_2^{\frac{t - \tau}{2}} \E\qty[\|\nabla F_\mu(\vtheta_{\tau - 1})\|] \overset{(a)}{=}& \frac1T \sum_{\tau=1}^T \sum_{t = \tau}^T \beta_2^{\frac{t - \tau}{2}} \E\qty[\|\nabla F_\mu(\vtheta_{\tau - 1})\|] = \frac1T \sum_{t=1}^T \frac{1 - \beta_2^{\frac{1+T-t}{2}}}{1 - \sqrt{\beta_2}} \E\qty[\|\nabla F_\mu(\vtheta_{t - 1})\|] \\ \overset{(b)}{\leq}& \frac{1}{1 - \beta_2}\frac1T \sum_{t=1}^T \E\qty[\|\nabla F_\mu(\vtheta_{t - 1})\|] \ ,
\end{aligned}
\end{equation}
where $(b)$ is due to the fact that $\beta_2 < 1$. In step $(a)$, we change the summation order over $t$ and $\tau$.


Therefore, \eqref{eq:termI_0} can be rewritten as:
\begin{equation}\label{eq:termI_1}
\begin{aligned}
    \frac{1}{T} \sum_{t=1}^T \E\qty[\sqrt{\sum_i^d \vv_{t, i}}] \leq& \sqrt{d \beta_2 \norm{\vv_{0}} + \frac{2 \beta_1 (1 + \beta_1)^2}{\qty(1 - \beta_1)^2 \qty(1 - \beta_2)} L_1^2 \eta^2 d^2 C_N} \\
    &\quad + \sqrt{\frac{2 \qty(1 + \beta_1) \qty(1 - \beta_2)}{\beta_2 - \beta_1^2}} \qty(1 - \beta_1) \qty(\sqrt{V} + \sqrt{\frac{L_0^2 \eta^2 d \qty(N^2 - 1)}{3 \qty(1 - \beta_2) N^2 K \mu^2}}) \\
    &\quad \quad + \sqrt{\frac{2}{\beta_1 \qty(1 - \beta_2)}} \qty(1 + \beta_1) \frac{1}{T} \sum_{t=1}^T \E\qty[\|\nabla F_\mu(\vtheta_{t - 1})\|] \\
    =& \Xi + \sqrt{\frac{2}{\beta_1 \qty(1 - \beta_2)}} \qty(1 + \beta_1) \frac{1}{T} \sum_{t=1}^T \E\qty[\|\nabla F_\mu(\vtheta_{t - 1})\|] \ ,
\end{aligned}
\end{equation}
where we let $\Xi \triangleq B_1 + \sqrt{\frac{2 \qty(1 + \beta_1) \qty(1 - \beta_1)^2 \beta_2}{\qty(\beta_2 - \beta_1^2) \qty(1 - \beta_2)} V}$ and $B_1 \triangleq \sqrt{d \beta_2 \norm{\vv_{0}} + \frac{2 \beta_1 (1 + \beta_1)^2}{\qty(1 - \beta_1)^2 \qty(1 - \beta_2)} L_1^2 \eta^2 d^2 C_N} + \sqrt{\frac{2 \qty(1 + \beta_1) \qty(1 - \beta_1)^2 \beta_2 L_0^2 \eta^2 d \qty(N^2 - 1)}{3 \qty(\beta_2 - \beta_1^2) \qty(1 - \beta_2)^2 N^2 K \mu^2}}$.

Therefore, we can bound Term I as follows:
\begin{equation}
\begin{aligned}
    \frac1T \sum_{t=1}^T \E\qty[\sqrt{\beta_2 \|\vv_t\| + \zeta}] \overset{(a)}{\leq}& \frac1T \sum_{t=1}^T \E\qty[\sqrt{\beta_2 \sum_i^d \vv_{t, i} + \zeta}] \overset{(b)}{\leq} \frac1T \sum_{t=1}^T \qty(\sqrt{\zeta} + \sqrt{\beta_2} \E\qty[\sqrt{\sum_i^d \vv_{t, i}}]) \\
    \overset{(c)}{\leq}& \sqrt{\zeta} + \Xi + \sqrt{\frac{2}{\beta_1 \qty(1 - \beta_2)}} \qty(1 + \beta_1) \frac{1}{T} \sum_{t=1}^T \E\qty[\|\nabla F_\mu(\vtheta_{t - 1})\|] \ ,
\end{aligned}
\end{equation}
where $(a)$ and $(b)$ are obtained by $\sqrt{\sum_i a_i} \leq \sum_i \sqrt{a_i}$, and $(c)$ is due to \eqref{eq:termI_1} and $\beta_2 \leq 1$.

\paragraph{Calculation of Term II.}
Following a similar approach as in \citep{radazo}, we first introduce the following auxiliary variable:
\begin{equation}
    \vx_t \triangleq \frac{\vtheta_t - \beta_1 / \sqrt{\beta_2} \vtheta_{t - 1}}{1 - \beta_1 / \sqrt{\beta_2}} = \frac{\vtheta_t - \kappa \vtheta_{t - 1}}{1 - \kappa} \ ,
\end{equation}
where $\kappa \triangleq \beta_1/\sqrt{\beta_2}$.

Based on the definition of $\vx_t$, the following relationships hold:
\begin{equation}\label{eq:x_theta_1}
    \vx_{t + 1} - \vx_t = \frac{\vtheta_{t + 1} - \vtheta_t - \kappa \qty(\vtheta_t - \vtheta_{t - 1})}{1 - \kappa} = \frac{1}{1 - \kappa} \qty(-\frac{\eta \vm_{t + 1}}{\sqrt{\vv_{t + 1} + \zeta}} + \kappa \frac{\eta \vm_t}{\sqrt{\vv_t + \zeta}}) \ ,
\end{equation}
and,
\begin{equation}\label{eq:x_theta_2}
    \vx_t - \vtheta_t = \frac{\kappa}{1 - \kappa} \qty(\vtheta_t - \vtheta_{t - 1}) = -\frac{\kappa}{1 - \kappa} \frac{\eta \vm_t}{\sqrt{\vv_t + \zeta}} \ .
\end{equation}

Starting from Lemma \ref{lem:Fmu_smoothness}:
\begin{equation}\label{eq:Lipschitz_2}
    F_\mu(\vx_{t + 1}) - F_\mu(\vx_t) \leq \aqty{\nabla F_\mu(\vx_t), \vx_{t + 1} - \vx_t}+ \frac{L_1}{2} \sqrt{d} \|\vx_{t + 1} - \vx_t\|^2 \ .
\end{equation}

Firstly, we focus on iteration $t$ and calculate the conditional expectation $\E\qty[\cdots | \gF_t]$ of \eqref{eq:Lipschitz_2}, where $\gF_t$ denotes all stochatsics up to $t$. After that:
\begin{equation}\label{eq:termII}
\begin{aligned}
    &\E\qty[F_\mu(\vx_{t + 1}) - F_\mu(\vx_t) | \gF_t] \leq \E\qty[\aqty{\nabla F_\mu(\vx_t), \vx_{t + 1} - \vx_t}| \gF_t] + \frac{L_1}{2} \sqrt{d} \E\qty[\|\vx_{t + 1} - \vx_t\|^2 | \gF_t] \\[3pt]
    =& \E\qty[\aqty{\nabla F_\mu(\vtheta_t), \vx_{t + 1} - \vx_t}| \gF_t] + \underbrace{\E\qty[\aqty{\nabla F_\mu(\vx_t) - \nabla F_\mu(\vtheta_t), \vx_{t + 1} - \vx_t}| \gF_t]}_{\circled{4}} + \underbrace{\frac{L_1}{2} \sqrt{d} \E\qty[\|\vx_{t + 1} - \vx_t\|^2 | \gF_t]}_{\circled{5}} \ .
\end{aligned}
\end{equation}

With the help of \eqref{eq:x_theta_1}, the first term of \eqref{eq:termII} can be separated as below:
\begin{equation}\label{eq:termII_1}
\begin{aligned}
    & \E\qty[\aqty{\nabla F_\mu(\vtheta_t), \vx_{t + 1} - \vx_t}| \gF_t] = \E\qty[\aqty{\nabla F_\mu(\vtheta_t), \frac{1}{1 - \kappa} \qty(-\frac{\eta \vm_{t + 1}}{\sqrt{\vv_{t + 1} + \zeta}} + \kappa \frac{\eta \vm_t}{\sqrt{\vv_t + \zeta}})} \Bigg| \gF_t] \\
    =& \frac{1}{1 - \kappa} \E\qty[\aqty{\nabla F_\mu(\vtheta_t), -\frac{\eta \vm_{t + 1}}{\sqrt{\beta_2 \vv_t + \zeta}} + \beta_1 \frac{\eta \vm_t}{\sqrt{\beta_2 \vv_t +\zeta}} -\frac{\eta \vm_{t + 1}}{\sqrt{\vv_{t + 1} + \zeta}} + \frac{\eta \vm_{t+1}}{\sqrt{\beta_2 \vv_t + \zeta}}} \Bigg| \gF_t] \\
    &\quad + \frac{1}{1 - \kappa} \E\qty[\aqty{\nabla F_\mu(\vtheta_t), \beta_1 \qty(\frac{\eta \vm_t}{\sqrt{\beta_2 \vv_t + \beta_2 \zeta}} - \frac{\eta \vm_t}{\sqrt{\beta_2 \vv_t + \zeta}})} \Bigg| \gF_t] \\
    =& \underbrace{\frac{1 - \beta_1}{1 - \kappa} \eta \E\qty[\aqty{\nabla F_\mu(\vtheta_t), -\frac{\hat{\nabla} f(\vtheta_t; \xi_{t+1})}{\sqrt{\beta_2 \vv_t + \zeta}}} \Bigg| \gF_t]}_{\circled{1}} \\
    &\quad + \underbrace{\frac{1}{1 - \kappa} \eta \E\qty[\aqty{\nabla F_\mu(\vtheta_t), \vm_{t + 1} \qty(\frac{1}{\sqrt{\beta_2 \vv_t + \zeta}} - \frac{1}{\sqrt{\vv_{t + 1} + \zeta}})} \Bigg| \gF_t]}_{\circled{2}} \\
    &\quad \quad + \underbrace{\frac{\beta_1}{1 - \kappa} \eta \E\qty[\aqty{\nabla F_\mu(\vtheta_t), \vm_t \qty(\frac{1}{\sqrt{\beta_2 \vv_t + \beta_2 \zeta}} - \frac{1}{\sqrt{\beta_2 \vv_t + \zeta}})} \Bigg| \gF_t]}_{\circled{3}} \ .
\end{aligned}
\end{equation}

Thereafter, we would bound term $\circled{1}$, $\circled{2}$, $\circled{3}$, $\circled{4}$ and $\circled{5}$ one by one.

For the term $\circled{1}$:
\begin{equation}\label{eq:term1}
\begin{aligned}
    \circled{1} =& \frac{1 - \beta_1}{1 - \kappa} \eta \E\qty[\aqty{\nabla F_\mu(\vtheta_t), -\frac{\hat{\nabla} f(\vtheta_t; \xi_{t+1})}{\sqrt{\beta_2 \vv_t + \zeta}}} \Bigg| \gF_t] \\
    =& -\frac{1 - \beta_1}{1 - \kappa} \eta \sum_i^d \nabla_i F_\mu(\vtheta_t) \frac{\E\qty[\hat{\nabla}_i f(\vtheta_t; \xi_{t+1}) \Big| \gF_t]}{\sqrt{\beta_2 \vv_{t, i} + \zeta}} \\
    =& -\frac{1 - \beta_1}{1 - \kappa} \eta \sum_i^d \qty(\frac{\qty|\nabla_i F_\mu(\vtheta_t)|^2}{\sqrt{\beta_2 \vv_{t, i} + \zeta}} + \frac{\nabla_i F_\mu(\vtheta_t)}{\sqrt{\beta_2 \vv_{t, i} + \zeta}} \E\qty[\hat{\nabla}_i f(\vtheta_t; \xi_{t+1}) - \nabla_i F_\mu(\vtheta_t) \Big| \gF_t]) \ .
\end{aligned}
\end{equation}

If we assume that the gradient of $F_\mu$ is bounded, i.e. $\qty|\nabla_i F_\mu(\vtheta_t)| \leq G_\mu$, we can simplify \ref{eq:term1} as follows:
\begin{equation}
    \circled{1} \leq -\frac{1 - \beta_1}{1 - \kappa} \eta \sum_i^d \frac{\qty|\nabla_i F_\mu(\vtheta_t)|^2}{\sqrt{\beta_2 \vv_{t, i} + \zeta}} + \frac{\qty(1 - \beta_1) G_\mu}{\qty(1 - \kappa) \sqrt{\zeta}} \eta \sum_i^d \E\qty[\qty|\hat{\nabla}_i f(\vtheta_t; \xi_{t+1}) - \nabla_i F_\mu(\vtheta_t)| \big| \gF_t] \ ,
\end{equation}
where we utilize $\sqrt{\beta_2 \vv_{t, i} + \zeta} \leq \sqrt{\zeta}$ and $a \leq |a|$ for simplicity.

Now turn to the calculation of term $\circled{2}$. Note that:
\begin{equation}\label{eq:term2_0}
\begin{aligned}
    &\frac{1}{\sqrt{\beta_2 \vv_{t, i} + \zeta}} - \frac{1}{\sqrt{\vv_{t + 1, i} + \zeta}} = \frac{\sqrt{\vv_{t + 1, i} + \zeta} - \sqrt{\beta_2 \vv_{t, i} + \zeta}}{\sqrt{\beta_2 \vv_{t, i} + \zeta} \sqrt{\vv_{t + 1, i} + \zeta}} \\
    \overset{(a)}{=}& \frac{\vv_{t + 1, i} - \beta_2 \vv_{t, i}}{\sqrt{\beta_2 \vv_{t, i} + \zeta} \sqrt{\vv_{t + 1, i} + \zeta} \qty(\sqrt{\vv_{t + 1, i} + \zeta} + \sqrt{\beta_2 \vv_{t, i} + \zeta})} \\
    =& \frac{\qty(1 - \beta_2) \vm_{t, i}^2}{\sqrt{\beta_2 \vv_{t, i} + \zeta} \sqrt{\vv_{t + 1, i} + \zeta} \qty(\sqrt{\vv_{t + 1, i} + \zeta} + \sqrt{\beta_2 \vv_{t, i} + \zeta})} \ ,
\end{aligned}
\end{equation}
where in step $(a)$, we multiply $\qty(\sqrt{\vv_{t + 1, i} + \zeta} + \sqrt{\beta_2 \vv_{t, i} + \zeta})$ in both the numerator and denominator.

Therefore, we can rewrite the term $\circled{2}$ as:
\begin{equation}\label{eq:term2_1}
\begin{aligned}
    \circled{2} =& \frac{1}{1 - \kappa} \eta \sum_i^d \E\qty[\aqty{\nabla_i F_\mu(\vtheta_t), \vm_{t + 1, i} \qty(\frac{1}{\sqrt{\beta_2 \vv_{t, i} + \zeta}} - \frac{1}{\sqrt{\vv_{t + 1, i} + \zeta}})} \Bigg| \gF_t] \\
    \overset{(a)}{=}& \frac{1}{1 - \kappa} \eta \sum_i^d \E\qty[\aqty{\nabla_i F_\mu(\vtheta_t), \vm_{t + 1, i} \frac{\qty(1 - \beta_2) \vm_{t, i}^2}{\sqrt{\beta_2 \vv_{t, i} + \zeta} \sqrt{\vv_{t + 1, i} + \zeta} \qty(\sqrt{\vv_{t + 1, i} + \zeta} + \sqrt{\beta_2 \vv_{t, i} + \zeta})}} \Bigg| \gF_t] \\
    \overset{(b)}{\leq}& \frac{1}{1 - \kappa} \eta \sum_i^d \E\qty[\qty| \nabla_i F_\mu(\vtheta_t)| \frac{\qty(1 - \beta_2) \vm_{t, i}^2 \qty| \vm_{t + 1, i}|}{\sqrt{\beta_2 \vv_{t, i} + \zeta} \sqrt{\vv_{t + 1, i} + \zeta} \qty(\sqrt{\vv_{t + 1, i} + \zeta} + \sqrt{\beta_2 \vv_{t, i} + \zeta})} \Bigg| \gF_t] \\
    \overset{(c)}{\leq}& \frac{1}{1 - \kappa} \eta \sum_i^d \E\qty[\qty| \nabla_i F_\mu(\vtheta_t)| \frac{\sqrt{1 - \beta_2} \vm_{t, i}^2}{\sqrt{\beta_2 \vv_{t, i} + \zeta} \qty(\sqrt{\vv_{t + 1, i} + \zeta} + \sqrt{\beta_2 \vv_{t, i} + \zeta})} \Bigg| \gF_t] \\
    =& \frac{1}{1 - \kappa} \eta \sum_i^d \frac{1}{\sqrt{\beta_2 \vv_{t, i} + \zeta}} \qty| \nabla_i F_\mu(\vtheta_t)| \E\qty[ \frac{\sqrt{1 - \beta_2} \vm_{t, i}^2}{\sqrt{\vv_{t + 1, i} + \zeta} + \sqrt{\beta_2 \vv_{t, i} + \zeta}} \Bigg| \gF_t] \\
    \overset{(d)}{\leq}& \frac{1}{1 - \kappa} \eta \sum_i^d \frac{1}{\sqrt{\beta_2 \vv_{t, i} + \zeta}} \qty(\frac{\qty| \nabla_i F_\mu(\vtheta_t)|^2}{2 \gamma_0}+ \frac{\gamma_0}{2} \qty(\E\qty[ \frac{\sqrt{1 - \beta_2} \vm_{t, i}^2}{\sqrt{\vv_{t + 1, i} + \zeta} + \sqrt{\beta_2 \vv_{t, i} + \zeta}} \Bigg| \gF_t])^2) \\
    \overset{(e)}{\leq}& \frac{1}{1 - \kappa} \eta \sum_i^d \qty(\frac{1}{2 \gamma_0} \frac{\qty| \nabla_i F_\mu(\vtheta_t)|^2}{\sqrt{\beta_2 \vv_{t, i} + \zeta}} + \frac{\gamma_0 \E\qty[\vm_{t, i}^2 | \gF_t]}{2 \sqrt{\beta_2 \vv_{t, i} + \zeta}}  \E\qty[ \frac{\sqrt{1 - \beta_2} \vm_{t, i}^2}{ \qty(\sqrt{\vv_{t + 1, i} + \zeta} + \sqrt{\beta_2 \vv_{t, i} + \zeta})^2} \Bigg| \gF_t]) \ ,
\end{aligned}
\end{equation}
where $(a)$ comes from \eqref{eq:term2_0}, $(b)$ is due to Cauchy-Schwarz inequality, $(c)$ is due to the fact that $\frac{\qty|\vm_{s, i}|}{\sqrt{\vv_{s, i} + \zeta}} \leq \frac{1}{\sqrt{1 - \beta_2}}$, $(d)$ is obtained by $ab \leq \frac{1}{\gamma_0} a^2 + \frac{\gamma_0}{2} b^2$ for any positive number $\gamma_0$, and $(e)$ results from the Hölder's inequality $\E[|\va \vb|]\leq\qty(\E\qty[|\va|^2])^{\frac{1}{2}}\qty(\E\qty[|\vb|^2])^{\frac{1}{2}}$.

Taking the Cauchy-Schwarz inequality and Assump. \ref{assump:1} into account, the term $\E\qty[\vm_{t, i}^2 | \gF_t]$ can be bounded by:
\begin{equation}\label{eq:abs_nabla_f}
    \qty|\hat{\nabla}_i f(\vtheta; \xi)| \leq \frac{d}{N K} \sum_{n, k}^{N, K} \qty| \frac{f\qty(\vtheta + \mu \vu_k; \xi) - b_t}{\mu} | \qty| \vu_{k, i} | \leq \frac{2 C d}{\mu} \ ,
\end{equation}
\begin{equation}\label{eq:abs_m}
    \qty| \vm_{t + 1, i}| = \qty| (1 - \beta_1) \sum_{\tau = 1}^t \beta_1^{t - \tau} \hat{\nabla}_i f(\vtheta_{\tau - 1}; \xi_\tau) | \leq (1 - \beta_1) \sum_{\tau = 1}^t \beta_1^{t - \tau} \qty| \hat{\nabla}_i f(\vtheta_{\tau - 1}; \xi_\tau) | \leq \frac{2 C d}{\mu} \ .
\end{equation}

Besides:
\begin{equation}\label{eq:term2_2}
\begin{aligned}
    &\E\qty[ \frac{\sqrt{1 - \beta_2} \vm_{t, i}^2}{\sqrt{\beta_2 \vv_{t, i} + \zeta} \qty(\sqrt{\vv_{t + 1, i} + \zeta} + \sqrt{\beta_2 \vv_{t, i} + \zeta})^2} \Bigg| \gF_t] \\
    \overset{(a)}{\leq}& \E\qty[\frac{\vv_{t + 1, i} + \zeta - \qty(\beta_2 \vv_{t, i} + \zeta)}{\sqrt{\vv_{t + 1, i} + \zeta} \sqrt{\beta_2 \vv_{t, i} + \zeta}\qty(\sqrt{\vv_{t + 1, i} + \zeta} + \sqrt{\beta_2 \vv_{t, i} + \zeta})} \Bigg| \gF_t] \\
    =& \E\qty[\frac{1}{\sqrt{\beta_2 \vv_{t, i} + \zeta}} - \frac{1}{\sqrt{\vv_{t + 1, i} + \zeta}} \Bigg| \gF_t] \ .
\end{aligned}
\end{equation}
where in step $(a)$ we apply $\qty(\sqrt{\vv_{t + 1, i} + \zeta} + \sqrt{\beta_2 \vv_{t, i} + \zeta}) \leq \sqrt{\vv_{t + 1, i} + \zeta}$.

Hence, substituting \eqref{eq:abs_m}, \eqref{eq:term2_2} into \eqref{eq:term2_1}:
\begin{equation}\label{eq:term2_3}
\begin{aligned}
    \circled{2} \leq& \frac{1}{1 - \kappa} \eta \sum_i^d \qty(\frac{1}{2 \gamma_0} \frac{\qty| \nabla_i F_\mu(\vtheta_t)|^2}{\sqrt{\beta_2 \vv_{t, i} + \zeta}} + \frac{\gamma_0}{2} \frac{4 C^2 d^2}{\mu} \E\qty[\frac{1}{\sqrt{\beta_2 \vv_{t, i} + \zeta}} - \frac{1}{\sqrt{\vv_{t + 1, i} + \zeta}} \Bigg| \gF_t]) \\
    =& \frac{1 - \beta_1}{4 \qty(1 - \kappa)} \eta \sum_i^d \frac{\qty| \nabla_i F_\mu(\vtheta_t)|^2}{\sqrt{\beta_2 \vv_{t, i} + \zeta}} + \frac{4C^2 d^2}{\qty(1 - \beta_1) \qty(1 - \kappa) \mu^2} \eta \sum_i^d \E\qty[\frac{1}{\sqrt{\beta_2 \vv_{t, i} + \zeta}} - \frac{1}{\sqrt{\vv_{t + 1, i} + \zeta}} \Bigg| \gF_t] \ ,
\end{aligned}
\end{equation}
where we let $\gamma_0 = \frac{2}{1 - \beta_1}$ in the last step.

Next, the term $\circled{3}$ can be bounded as below:
\begin{equation}\label{eq:term3}
\begin{aligned}
    \circled{3} =& \frac{\beta_1}{1 - \kappa} \eta \sum_i^d \nabla_i F_\mu(\vtheta_t) \vm_{t, i} \qty(\frac{1}{\sqrt{\beta_2 \vv_{t, i} + \beta_2 \zeta}} - \frac{1}{\sqrt{\beta_2 \vv_{t, i} + \zeta}}) \\
    \overset{(a)}{\leq}& \frac{\beta_1}{1 - \kappa} \eta \sum_i^d \qty|\nabla_i F_\mu(\vtheta_t)| \qty|\vm_{t, i}| \qty|\frac{1}{\sqrt{\beta_2 \vv_{t, i} + \beta_2 \zeta}} - \frac{1}{\sqrt{\beta_2 \vv_{t, i} + \zeta}}| \\
    =& \frac{\beta_1}{1 - \kappa} \eta \sum_i^d \qty|\nabla_i F_\mu(\vtheta_t)| \qty|\vm_{t, i}| \qty|\frac{\qty(1 - \beta_2) \zeta}{\sqrt{\beta_2 \vv_{t, i} + \beta_2 \zeta} \sqrt{\beta_2 \vv_{t, i} + \zeta} \qty(\sqrt{\beta_2 \vv_{t, i} + \beta_2 \zeta} + \sqrt{\beta_2 \vv_{t, i} + \zeta})}| \\
    \overset{(b)}{\leq}& \frac{\beta_1}{1 - \kappa} \eta \sum_i^d \frac{1}{\sqrt{\beta_2 \vv_{t, i} + \zeta}} \qty|\nabla_i F_\mu(\vtheta_t)| \qty|\frac{\sqrt{1 - \beta_2} \zeta}{\sqrt{\beta_2} \qty(\sqrt{\beta_2 \vv_{t, i} + \beta_2 \zeta} + \sqrt{\beta_2 \vv_{t, i} + \zeta})}| \\
    \overset{(c)}{\leq}& \frac{\beta_1}{1 - \kappa} \eta \sum_i^d \qty(\frac{\qty|\nabla_i F_\mu(\vtheta_t)|^2}{2 \gamma_1 \sqrt{\beta_2 \vv_{t, i} + \zeta}} + \frac{\gamma_1 \qty(1 - \beta_2) \zeta^2}{2 \beta_2 \sqrt{\beta_2 \vv_{t, i} + \zeta} \qty(\sqrt{\beta_2 \vv_{t, i} + \beta_2 \zeta} + \sqrt{\beta_2 \vv_{t, i} + \zeta})^2}) \\
    \overset{(d)}{\leq}& \frac{\beta_1}{1 - \kappa} \eta \sum_i^d \qty(\frac{\qty|\nabla_i F_\mu(\vtheta_t)|^2}{2 \gamma_1 \sqrt{\beta_2 \vv_{t, i} + \zeta}} + \frac{\gamma_1 \qty(1 - \beta_2) \sqrt{\zeta}}{8 \beta_2^2}) \\
    =& \frac{1 - \beta_1}{4 \qty(1 - \kappa)} \eta \sum_i^d \frac{\qty|\nabla_i F_\mu(\vtheta_t)|^2}{\sqrt{\beta_2 \vv_{t, i} + \zeta}} + \frac{\beta_1^2 \qty(1 - \beta_2)}{4 \beta_2^2 \qty(1 - \beta_1) \qty(1 - \kappa)} \eta d \sqrt{\zeta} \\
    \overset{(e)}{\leq}& \frac{1 - \beta_1}{4 \qty(1 - \kappa)} \eta \sum_i^d \frac{\qty|\nabla_i F_\mu(\vtheta_t)|^2}{\sqrt{\beta_2 \vv_{t, i} + \zeta}} + \frac{\beta_1^2 \qty(1 - \beta_2)}{\qty(1 - \beta_1) \qty(1 - \kappa)} \eta d \sqrt{\zeta} \ ,
\end{aligned}
\end{equation}
where $(a)$ comes from Cauchy-Schwarz inequality, $(b)$ results from the fact that $\frac{\vm_{t, i}^2}{\vv_{t, i} + \zeta} \leq \frac{1}{1 - \beta_2}$, $(c)$ is because of $ab \leq \half (a^2 + b^2)$, and $(d)$ is due to $\sqrt{\vv_{t, i} + \zeta} \leq \sqrt{\zeta}$. In step $(e)$, we assume $2 \beta_2 \geq 1$ and let $\gamma_1 = \frac{2 \beta_1}{1 - \beta_1}$.

Term $\circled{4}$ is bounded as below:
\begin{equation}\label{eq:term4}
\begin{aligned}
    \circled{4} =& \sum_i^d \E\qty[\qty(\nabla_i F_\mu(\vx_t) - \nabla_i F_\mu(\vtheta_t))\qty(\vx_{t + 1, i} - \vx_{t, i})] \\
    \overset{(a)}{\leq}& \frac{1}{1 - \kappa} \sum_i^d \E\qty[\qty|\nabla_i F_\mu(\vx_t) - \nabla_i F_\mu(\vtheta_t)|\qty|\vtheta_{t + 1, i} - \vtheta_{t, i} - \kappa \qty(\vtheta_{t, i} - \vtheta_{t - 1, i})| | \gF_t] \\
    \overset{(b)}{\leq}& \frac{\kappa}{\qty(1 - \kappa)^2} L_1 \sum_i^d \E\qty[\norm{\vtheta_t - \vtheta_{t - 1}} \qty|\vtheta_{t + 1, i} - \vtheta_{t, i} - \kappa \qty(\vtheta_{t, i} - \vtheta_{t - 1, i})| | \gF_t] \\
    \overset{(c)}{\leq}& \frac{\kappa}{\qty(1 - \kappa)^2} L_1 \sum_i^d \E\qty[\norm{\vtheta_t - \vtheta_{t - 1}} \qty|\vtheta_{t + 1, i} - \vtheta_{t, i}| + \kappa \norm{\vtheta_t - \vtheta_{t - 1}} \qty|\vtheta_{t, i} - \vtheta_{t - 1, i}| | \gF_t] \\
    \overset{(d)}{\leq}& \frac{\kappa}{2 \qty(1 - \kappa)^2} \sqrt{d} L_1 \eta^2 \sum_i^d \qty(\qty(1 + 2 \kappa) \frac{\vm_{t, i}^2}{\vv_{t, i} + \zeta} + \E\qty[\frac{\vm_{t + 1, i}^2}{\vv_{t + 1, i} + \zeta} \Bigg| \gF_t]) \\
    \overset{(e)}{\leq}& \frac{1}{2 \qty(1 - \kappa)^2} \sqrt{d} L_1 \eta^2 \sum_i^d \qty(3 \frac{\vm_{t, i}^2}{\vv_{t, i} + \zeta} + \E\qty[\frac{\vm_{t + 1, i}^2}{\vv_{t + 1, i} + \zeta} \Bigg| \gF_t]) \ ,
\end{aligned}
\end{equation}
where $(a)$ is due to \eqref{eq:x_theta_1} and Cauchy-Schwarz inequality, $(b)$ is due to Lemma \ref{lem:Fmu_smoothness}, $(c)$ comes from the fact that $|a - b| \leq |a| + |b|$, and in step $(e)$ we assume $\kappa \leq 1$. In step $(d)$, we apply the following inequality by $ab \leq \frac{a^2}{2 \sqrt{d}} + \frac{\sqrt{d} b^2}{2}$:
\begin{equation}
\begin{aligned}
    &\sum_i^d \qty(\norm{\vtheta_t - \vtheta_{t - 1}} \qty|\vtheta_{t + 1, i} - \vtheta_{t, i}| + \kappa \norm{\vtheta_t - \vtheta_{t - 1}} \qty|\vtheta_{t, i} - \vtheta_{t - 1, i}|) \\
    \leq& \sum_i^d \qty(\frac{\norm{\vtheta_t - \vtheta_{t - 1}}^2}{2 \sqrt{d}} + \frac{\sqrt{d} \qty|\vtheta_{t + 1, i} - \vtheta_{t, i}|^2}{2} + \kappa \qty(\frac{\norm{\vtheta_t - \vtheta_{t - 1}}^2}{2 \sqrt{d}} + \frac{\sqrt{d} \qty|\vtheta_{t, i} - \vtheta_{t - 1, i}|^2}{2})) \\
    =& \frac{\sqrt{d}}{2} \qty(\qty(1 + 2 \kappa) \norm{\vtheta_t - \vtheta_{t - 1}}^2 + \norm{\vtheta_{t + 1} - \vtheta_{t}}^2) \ .
\end{aligned}
\end{equation}

Finally, with the help of \eqref{eq:x_theta_1}, the term $\circled{5}$ is bounded as below:
\begin{equation}\label{eq:term5}
\begin{aligned}
    \circled{5} =& \frac{1}{2 \qty(1 - \kappa)^2} \sqrt{d} L_1 \sum_i^d \E\qty[\qty|\vtheta_{t + 1, i} - \vtheta_{t, i} - \kappa \qty(\vtheta_{t, i} - \vtheta_{t - 1, i})|^2 \Big| \gF_t] \\
    \overset{(a)}{\leq}& \frac{1}{\qty(1 - \kappa)^2} \sqrt{d} L_1 \sum_i^d \E\qty[\qty|\vtheta_{t + 1, i} - \vtheta_{t, i}|^2 + \kappa^2 \qty|\vtheta_{t, i} - \vtheta_{t - 1, i}| \Big| \gF_t] \\
    =& \frac{1}{\qty(1 - \kappa)^2} \sqrt{d} L_1 \eta^2 \sum_i^d \qty(\E\qty[\frac{\vm_{t + 1, i}^2}{\vv_{t + 1, i} + \zeta} \Bigg| \gF_t] + \kappa^2 \frac{\vm_{t, i}^2}{\vv_{t, i} + \zeta})  \\
    \overset{(b)}{\leq}& \frac{1}{\qty(1 - \kappa)^2} \sqrt{d} L_1 \eta^2 \sum_i^d \qty(\E\qty[\frac{\vm_{t + 1, i}^2}{\vv_{t + 1, i} + \zeta} \Bigg| \gF_t] + \frac{\vm_{t, i}^2}{\vv_{t, i} + \zeta}) \ ,
\end{aligned}
\end{equation}
where $(a)$ results from the inequality $\qty(a - b)^2 \leq 2 a^2 + 2 b^2$, and in step $(b)$, we assume $\kappa \leq 1$.

Gathering the results of \eqref{eq:term1}, \eqref{eq:term2_3}, \eqref{eq:term3}, \eqref{eq:term4} and \eqref{eq:term5}, \eqref{eq:termII} can be bounded as below:
\begin{equation}\label{eq:termII_final}
\begin{aligned}
    \E\qty[F_\mu(\vx_{t + 1}) - F_\mu(\vx_t) | \gF_t] 
    &\leq -\frac{1 - \beta_1}{2 \qty(1 - \kappa)} \eta \sum_i^d \frac{\qty|\nabla_i F_\mu(\vtheta_t)|^2}{\sqrt{\beta_2 \vv_{t, i} + \zeta}} + \frac{\beta_1^2 \qty(1 - \beta_2)}{\qty(1 - \beta_1) \qty(1 - \kappa)} \eta d \sqrt{\zeta} \\
    &\quad + \frac{4 \eta C^2 d^3}{\qty(1 - \beta_1) \qty(1 - \kappa) \mu^2} \sum_i^d \E\qty[\frac{1}{\sqrt{\beta_2 \vv_{t, i} + \zeta}} - \frac{1}{\sqrt{\vv_{t + 1, i} + \zeta}}  \Bigg| \gF_t] \\
    &\quad \quad + \frac{5 L_1 \sqrt{d}}{2 \qty(1 - \kappa)^2} \eta^2 \sum_i^d \frac{\vm_{t, i}^2}{\vv_{t, i} + \zeta} + \frac{3 L_1 \sqrt{d}}{2 \qty(1 - \kappa)^2} \eta^2 \sum_i^d \E\qty[\frac{\vm_{t + 1, i}^2}{\vv_{t + 1, i} + \zeta} \Bigg| \gF_t] \\
    &+ \frac{\qty(1 - \beta_1) G_\mu}{\qty(1 - \kappa) \sqrt{\zeta}} \eta \sum_i^d \E\qty[\qty|\hat{\nabla}_i f(\vtheta_t; \xi_{t+1}) - \nabla_i F_\mu(\vtheta_t)| \big| \gF_t] \ .
\end{aligned}
\end{equation}

Considering the summation of \eqref{eq:termII_final} over all iterations $t$ from $0$ to $T - 1$:
\begin{equation}
    \text{LHS} = \sum_{t = 0}^{T - 1} \E\qty[F_\mu(\vx_{t + 1}) - F_\mu(\vx_t)] = \E\qty[F_\mu(\vx_T)] - F_\mu(\vx_0) \triangleq - \Delta \ ,
\end{equation}
\begin{equation}\label{eq:termII_RHS}
\begin{aligned}
    \text{RHS} \overset{(a)}{\leq}& -\frac{1 - \beta_1}{2 \qty(1 - \kappa)} \eta \sum_{t = 0}^{T - 1} \sum_i^d \frac{\qty|\nabla_i F_\mu(\vtheta_t)|^2}{\sqrt{\beta_2 \vv_{t, i} + \zeta}} + \frac{\beta_1^2 \qty(1 - \beta_2)}{\qty(1 - \beta_1) \qty(1 - \kappa)} T d \sqrt{\zeta} \\
    &\quad + \frac{4 C^2 d^3}{\qty(1 - \beta_1) \qty(1 - \kappa) \mu^2} \qty(\frac{1}{\sqrt{\zeta}} + \frac{T \qty(1 - \beta_2)}{\sqrt{\zeta}}) + \frac{\qty(1 - \beta_1)}{2 \qty(1 - \kappa)} \eta T G \\
    &\quad \quad + \frac{4 L_1 \sqrt{d}}{\qty(1 - \kappa)^2} \eta^2 \sum_i^d \qty( \frac{1}{1 - \beta_2} \ln \qty(\frac{\beta_2^T \vv_{0, i} + 4 C^2 d^2 / \mu^2}{\vv_{0, i}}) + 2T) \ ,
\end{aligned}
\end{equation}
where $G \triangleq \frac{2 G_\mu}{\sqrt{\zeta}} \sqrt{d \qty(V + \frac{L_0^2 \eta^2 d \qty(N^2 - 1)}{3 \qty(1 - \beta_2) N^2 K \mu^2} + \frac{L_1^2 \eta^2 d^2 \qty(N - 1)}{2 \qty(1 - \beta_2)})}$ is a constant number. In step $(a)$, we apply the following three inequalities. The first one is:
\begin{equation}
\begin{aligned}
    &\sum_{t = 0}^T \sum_i^d \E\qty[\frac{1}{\sqrt{\beta_2 \vv_{t, i} + \zeta}} - \frac{1}{\sqrt{\vv_{t + 1, i} + \zeta}}] \\
    =& \sum_i^d \qty(\frac{1}{\sqrt{\beta_2 \vv_{0, i} + \zeta}} + \sum_{t = 0}^{T - 2} \E\qty[\frac{1}{\sqrt{\beta_2 \vv_{t + 1, i} + \zeta}} - \frac{1}{\sqrt{\vv_{t + 1, i} + \zeta}}] - \E\qty[\frac{1}{\sqrt{\vv_{T, i} + \zeta}}]) \\
    \leq& \sum_i^d \qty(\frac{1}{\sqrt{\zeta}} + \sum_{t = 0}^{T - 2} \E\qty[\frac{1}{\sqrt{\beta_2 \vv_{t + 1, i} + \zeta}} - \frac{1}{\sqrt{\vv_{t + 1, i} + \zeta}}]) \\
    =& \sum_i^d \qty(\frac{1}{\sqrt{\zeta}} + \sum_{t = 0}^{T - 2} \E\qty[\frac{\qty(1 - \beta_2) \vv_{t+1, i}}{\sqrt{\beta_2 \vv_{t + 1, i} + \zeta} \sqrt{\vv_{t + 1, i} + \zeta} \qty(\sqrt{\beta_2 \vv_{t + 1, i} + \zeta} + \sqrt{\vv_{t + 1, i} + \zeta})}]) \\
    \leq& \sum_i^d \qty(\frac{1}{\sqrt{\zeta}} + \frac{1 - \beta_2}{\sqrt{\zeta}} T) \ .
\end{aligned}
\end{equation}

The second one is:
\begin{equation}
\begin{aligned}
    \sum_{t = 0}^{T - 1} \frac{\qty(1 - \beta_2) \vm_{t, i}^2}{\vv_{t, i} + \zeta} =& \sum_{t = 0}^{T - 1}\frac{\frac{\qty(1 - \beta_2) \vm_{t, i}^2}{\vv_{t, i} - \qty(1 - \beta_2) \vm_{t, i}^2}}{1 + \frac{\qty(1 - \beta_2) \vm_{t, i}^2}{\vv_{t, i} - \qty(1 - \beta_2) \vm_{t, i}^2}} \leq \sum_{t = 0}^{T - 1} \ln \qty(1 + \frac{\qty(1 - \beta_2) \vm_{t, i}^2}{\vv_{t, i} - \qty(1 - \beta_2) \vm_{t, i}^2}) \\
    =& \sum_{t = 0}^{T - 1} \ln \qty(\frac{\vv_{t, i}}{\beta_2 \vv_{t - 1, i}}) = \ln \qty(\frac{\vv_{T, i}}{\vv_{0, i}}) - T \ln \beta_2 \ ,
\end{aligned}
\end{equation}
where we utilize $\ln (1 + a) \leq a$. 

And the last one is:
\begin{equation}
\begin{aligned}
    &\sum_{t = 0}^{T - 1} \sum_i^d \E\qty[\qty|\hat{\nabla}_i f(\vtheta_t; \xi_{t+1}) - \nabla_i F_\mu(\vtheta_t)|] \overset{(a)}{\leq} \sqrt{d} \sum_{t = 0}^{T - 1} \E\qty[\norm{\hat{\nabla} f(\vtheta_t; \xi_{t+1}) - \nabla F_\mu(\vtheta_t)}] \\
    \overset{(b)}{\leq}& \sqrt{d} \sum_{t = 0}^{T - 1} \sqrt{\E\qty[\norm{\hat{\nabla} f(\vtheta_t; \xi_{t+1}) - \nabla F_\mu(\vtheta_t)}^2]} \overset{(c)}{\leq} T \sqrt{d \qty(V + \frac{L_0^2 \eta^2 d \qty(N^2 - 1)}{3 \qty(1 - \beta_2) N^2 K \mu^2} + \frac{L_1^2 \eta^2 d^2 \qty(N - 1)}{2 \qty(1 - \beta_2)})} \ ,
\end{aligned}
\end{equation}
where $(a)$ comes from Cauchy-Schwarz inequality, $(b)$ is due to Jensen's inequality, and $(c)$ comes from Thm. \ref{thm:bias-variance-decomp}.

Reorganizing \eqref{eq:termII_RHS}, we can derive:
\begin{equation}
\begin{aligned}
    &\frac1T \sum_{t = 0}^{T - 1} \sum_i^d \frac{\qty|\nabla_i F_\mu(\vtheta_t)|^2}{\sqrt{\beta_2 \vv_{t, i} + \zeta}} \leq \frac{8 C^2 d^3}{\qty(1 - \beta_1)^2 \mu^2 \eta T} \qty(\frac{1}{\sqrt{\zeta}} + \frac{T \qty(1 - \beta_2)}{\sqrt{\zeta}}) + \frac{\beta_1^2 \qty(1 - \beta_2)}{\qty(1 - \beta_1)^2} d \sqrt{\zeta} \\
    &\quad + \frac{2 \qty(1 - \kappa)}{\qty(1 - \beta_1) \eta T} \Delta + \frac{8 L_1 \sqrt{d}}{\qty(1 - \beta_1) \qty(1 - \kappa) T} \eta \sum_i^d \qty(\frac{1}{1 - \beta_2} \ln \qty(\frac{\beta_2^T \vv_{0, i} + 4 C^2 d^2 / \mu^2}{\vv_{0, i}}) + 2T) + G \ .
\end{aligned}
\end{equation}

To simplify the equation, we choose $1 - \beta_2 = \min \qty(\frac{\qty(1 - \beta_1)^2 \mu^2 \eta \sqrt{\zeta} \epsilon^2}{64 C^2 d^3}, \frac{\qty(1 - \beta_1)^2 \epsilon^2}{4 \beta_1^2 d \sqrt{\zeta}}) \sim \gO\qty(\epsilon^2)$, $T = \max \qty(\frac{64 C^2 d^3}{\qty(1 - \beta_1)^2 \mu^2 \eta \sqrt{\zeta} \eta \epsilon^2}, \frac{8 \qty(1 - \kappa)}{\qty(1 - \beta_1) \eta \epsilon^2}, \frac{64 L_1 \sqrt{d} \eta}{\qty(1 - \beta_1) \qty(1 - \kappa) \qty(1 - \beta_2) \epsilon^2} \sum_i^d \ln \qty(\frac{\beta_2^T \vv_{0, i} + 4 C^2 d^2 / \mu^2}{\vv_{0, i}})) \sim \gO\qty(\epsilon^{-4})$, $\eta = \frac{\qty(1 - \beta_1) \qty(1 - \kappa) \epsilon^2}{128 L_1 d^{3 / 2}} \sim \gO\qty(\epsilon^2)$, and then have:
\begin{equation}\label{eq:termII_eps}
    \frac1T \sum_{t = 0}^{T - 1} \frac{\norm{\nabla F_\mu(\vtheta_t)}^2}{\sqrt{\beta_2 \norm{\vv_t} + \zeta}} \leq \frac1T \sum_{t = 0}^{T - 1} \sum_i^d \frac{\qty|\nabla_i F_\mu(\vtheta_t)|^2}{\sqrt{\beta_2 \vv_{t, i} + \zeta}} \leq \frac{1}{4} \epsilon^2 + \frac{1}{4} \epsilon^2 + \frac{1}{4} \epsilon^2 + \frac{1}{4} \epsilon^2 + G \leq \epsilon^2 + G \ .
\end{equation}

Overall, inserting the results of term I \eqref{eq:termI_0} and term II \eqref{eq:termII_eps} into \eqref{eq:1}:
\begin{equation}
    \qty(\frac1T \sum_{t = 0}^{T - 1} \E\qty[\norm{\nabla F_\mu(\vtheta_t)}])^2 \leq \qty(\sqrt{\zeta} + \Xi + \sqrt{\frac{2}{\beta_1 \qty(1 - \beta_2)}} \qty(1 + \beta_1) \frac{1}{T} \sum_{t=1}^T \E\qty[\|\nabla F_\mu(\vtheta_{t - 1})\|]) \qty(\epsilon^2 + G) \ ,
\end{equation}
which is actually a quadratic inequality. After solving the root of the quadratic equation, we obtain:
\begin{equation}
\begin{aligned}
    \frac1T \sum_{t = 0}^{T - 1} \E\qty[\norm{\nabla F_\mu(\vtheta_t)}] \leq& \sqrt{\frac{2}{\beta_1 \qty(1 - \beta_2)}} \qty(1 + \beta_1) \qty(\epsilon^2 + G) + \qty(\sqrt[4]{\zeta} + \sqrt{\Xi}) \sqrt{\epsilon^2 + G} \\
    \leq& \sqrt{\frac{2}{\beta_1 \qty(1 - \beta_2)}} \qty(1 + \beta_1) \qty(\epsilon^2 + G) + \qty(\sqrt[4]{\zeta} + \sqrt{\Xi}) \qty(\epsilon + \sqrt{G}) \\
    \leq& \sqrt{\frac{2}{\beta_1 \qty(1 - \beta_2)}} \qty(1 + \beta_1) \epsilon^2 + \qty(\sqrt[4]{\zeta} + \sqrt{\Xi}) \epsilon + B_2 \ ,
\end{aligned}
\end{equation}
where $B_2 \triangleq \sqrt{\frac{2}{\beta_1 \qty(1 - \beta_2)}} \qty(1 + \beta_1) G + \qty(\sqrt[4]{\zeta} + \sqrt{\Xi}) \sqrt{G}$.

To derive the convergence guarantee of $F(\vtheta_t)$, we introduce the bias between $F_\mu(\vtheta_t)$ and $F(\vtheta_t)$, which is defined as:
\begin{equation}
\begin{aligned}
    \E\qty[\norm{\nabla F(\vtheta) - \nabla F_\mu(\vtheta)}] \overset{(a)}{=}& \E\qty[\norm{\E_\rvu\qty[\nabla F(\vtheta) - \nabla F(\vtheta + \mu \rvu)]}] \overset{(b)}{\leq} \E\qty[\norm{\nabla F(\vtheta) - \nabla F(\vtheta + \mu \rvu)}] \\
    \overset{(c)}{\leq}& \sqrt{d} L_1 \E\qty[\norm{\mu \rvu}] \overset{(d)}{=} \mu L_1 \sqrt{d} \ ,
\end{aligned}
\end{equation}
where $(a)$ comes from the definition of $F_\mu$ \eqref{eq:fu}, $(b)$ results from Jensen's inequility, $(c)$ is due to Assump. \ref{assump:1}, and $(d)$ follows from the fact that $\rvu \sim  \mathrm{Unif}(\sS^{d - 1})$ and hence $\norm{\rvu} = 1$.

Afterthat, we can bound the convergence of $F(\vtheta_t)$ as below:
\begin{equation}
\begin{aligned}
    \frac1T \sum_{t = 0}^{T - 1} \E\qty[\norm{\nabla F(\vtheta_t)}] \leq& \frac1T \sum_{t = 0}^{T - 1} \E\qty[\norm{\nabla F_\mu(\vtheta_t)}] + \frac1T \sum_{t = 0}^{T - 1} \E\qty[\norm{\nabla F_\mu(\vtheta_t) - \nabla F(\vtheta_t)}] \\
    \leq& \sqrt{\frac{2}{\beta_1 \qty(1 - \beta_2)}} \qty(1 + \beta_1) \epsilon^2 + \qty(\sqrt[4]{\zeta} + \sqrt{\Xi}) \epsilon + \mu L_1 \sqrt{d} + B_2 \ ,
\end{aligned}
\end{equation}
which completes the proof.
\end{proof}

\section{Experiments Setup}\label{appx:exp_setup}

In this section, we first introduce the baselines used in our experiments (Sec. \ref{appx:baselines}), and then we provide experimental details on synthetic functions (Sec. \ref{subsec:synthetic_details}), black-box adversarial attack (Sec. \ref{subsec:adversarial_details}), and memory-efficient LLM fine-tuning (Sec. \ref{subsec:mezo_details}).

\subsection{Baselines}\label{appx:baselines}
First of all, we claims that our experiments compare only the differing gradient estimation methods among all baselines and \ours{}. Consequently, all baselines and \ours{} share the same update rule, such as ZO-AdaMM and $\gR$-AdaZO. Below, we introduce the three baselines used in our study.
\begin{itemize}
[topsep=0pt,leftmargin=10pt,itemsep=0pt]
    \item \textbf{Vanilla ZOO}. This zeroth-order optimization algorithm employs the gradient estimator in \eqref{eq:fd}. When paired with the Adam update rule, it is denoted ZO-AdaMM \citep{zo-adamm}; when paired with the $\gR$-AdaZO update rule, it is referred to as $\gR$-AdaZO \citep{radazo}.
    \item \textbf{ReLIZO} \citep{relizo}. ReLIZO is zeroth-order gradient estimation algorithm, which reuses queries from previous iterations through a quadratically constrained linear program, and effectively decouples sample size from variable dimension.
    \item \textbf{ZOO with HiStorical gradient (ZoHS)}. On the basis of the Vanilla ZOO framework, ZoHS integrates historical gradient information into the gradient estimation procedure. Specifically, the gradeint estimator for ZoHS is formally defined as:
    \begin{equation}
        \hat{\nabla}_{\text{ZoHS}} F(\vtheta_{t - 1}) \triangleq \frac{1}{N} \sum_{n = 1}^N \hat{\nabla} F(\vtheta_{t - n}) \ ,
    \end{equation}
    where $\hat{\nabla} F(\vtheta_{t - n})$ is the gradient estimator of Vanilla ZOO at iteration $t - n$.
\end{itemize}

\subsection{Synthetic Functions}\label{subsec:synthetic_details}

All experiments are conducted in $d = 10000$ dimensions and run for $T = 20000$ iterations. For a fair comparison, all experiments share the same initialization and hyperparameters: the step size $\eta = 0.001$, the number of queries $K = 10$, the smoothing radius parameter $\mu = 0.05$, and the number of histories $N = 6$. The analytical forms of the synthetic functions used in our experiments are as follows:

\begin{itemize}[leftmargin=10pt]
\item \textbf{Ackley Function:}
\begin{equation}
    f(\vtheta) = -20 \exp\qty(-0.2 \sqrt{\frac{1}{d} \sum_{i=1}^D \theta_i^2}) - \exp\qty(\frac{1}{d} \sum_{i=1}^d \cos(2\pi \theta_i)) + 20 + e \ .
\end{equation}

\item \textbf{Levy Function:}
\begin{equation}
    f(\vtheta) = \sin^2(\pi w_1) + \sum_{i=1}^{d-1} \qty(w_i - 1)^2 \qty[1 + 10 \sin^2(\pi w_i + 1)] + \qty(w_d - 1)^2 \qty[1 + \sin^2(2\pi w_d)] \ ,
\end{equation}
where $w_i = 1 + \frac{\theta_i-1}{4}$.

\item \textbf{Quadratic Function:}
\begin{equation}
    f(\vtheta) = \half \sum_{i=1}^d \theta_i^2 \ .
\end{equation}

\item \textbf{Rosenbrock Function:}
\begin{equation}
    f(\vtheta) = \sum_{i=1}^{d-1} \qty[100 (\theta_{i + 1} - \theta_i^2)^2 + (1 - \theta_i)^2] \ .
\end{equation}
\end{itemize}
Note that all four functions have the same optimal solution of zero.

\subsection{Black-box Adversarial Attack}\label{subsec:adversarial_details}
For the black-box adversarial attack, we use the same model as in \citep{relizo}: a simple two-layer CNN trained on the MNIST dataset. To ensure a fair comparison, all experiments utilize the same initialization and the following hyperparameters: step size $\eta = 0.01$, number of queries $K = 2$, smoothing parameter $\mu = 0.5$, and number of histories $N = 6$.

\subsection{Memory-Efficient LLM Fine-Tuning}\label{subsec:mezo_details}
For the memory-efficient fine-tuning of large language models, we select OPT-1.3B and OPT-13B \citep{zhang2022opt} as the pretrained models, and fine-tune them with LoRA adapters on the SST-2 and COPA datasets from the GLUE benchmark \citep{wang2019glue}. All experiments are conducted using the same initialization and hyperparameters: step size $\eta = 0.00005$, number of queries $K = 2$, smoothing parameter $\mu = 0.01$, and history lengths $N = \{15, 50\}$. The batch size is fixed at $16$ for both datasets.

\section{Additional Experiments}\label{sec:additional_experiments}
\subsection{The Equivalence between ZOO and \rein{}}
To empirically validate our core theoretical finding that the Gaussian-smoothed ZOO shares the same convergence as the single-step REINFORCE with baseline (Cor.~\ref{cor:equiv-conv}), we conduct comparison on four synthetic functions. Fig.~\ref{fig:synthetic-baseline} illustrates these comparisons using two baselines: the standard ZOO single-point baseline ($b_t = f(\vtheta_{t-1}; \xi)$, red curves) and an averaged baseline (green curves) proposed for \ours{} in \eqref{eq:baseline-avg}. The results in Fig.~\ref{fig:synthetic-baseline} clearly deliver two key points. First, for any given baseline strategy (either single-point or averaged), the convergence trajectories of ZOO and REINFORCE are virtually indistinguishable across all four synthetic functions. This provides strong numerical evidence supporting our theoretical equivalence. Second, the averaged baseline (green curves) consistently and significantly outperforms the single-point baseline (red curves) for both ZOO and REINFORCE. This manifests as faster convergence and a lower final optimality gap, underscoring the effectiveness of the PO-inspired averaged baseline in reducing variance and improving optimization performance, a central premise of our \ours{}.
\begin{figure*}[t]
    \centering
    \includegraphics[width=1.0\textwidth]{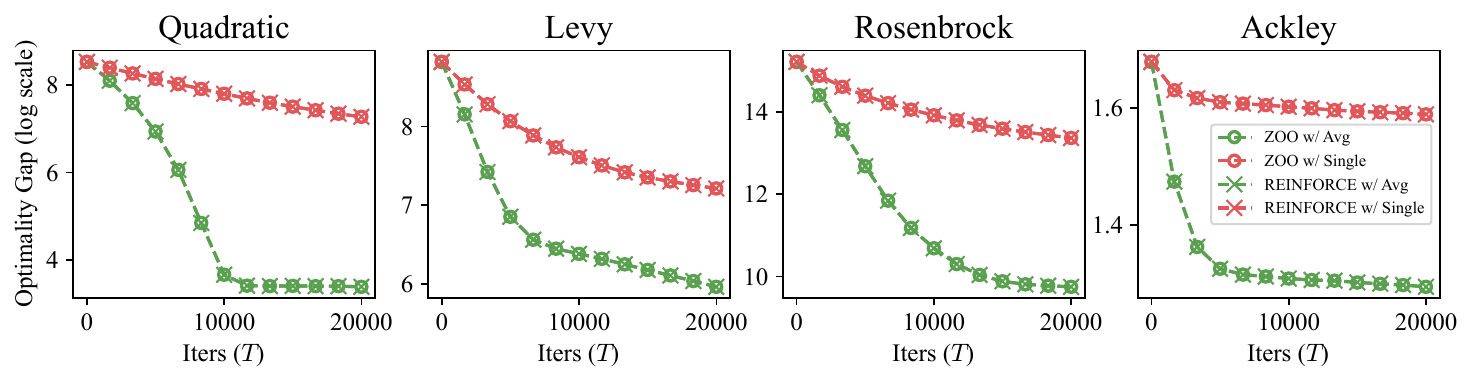}
    \vspace{-7mm}
    \caption{Equivalence of ZOO and \rein{} with two different baselines. The $y$-axis denotes the gap between the current function value and the optimal function value. The green curves denote to the average baseline defined in \eqref{eq:baseline-avg}, while the red curves denote to the single-point baseline $b_t = f(\vtheta_{t-1}; \xi)$. All curves are averaged over 5 independent runs.}
    \label{fig:synthetic-baseline}
    \vspace{-3mm}
\end{figure*}

\subsection{Synthetic Functions Optimization under $\gR$-AdaZO}\label{subsec:synthetic_radazo}
Consistent with the experiments in Section \ref{sec:synthetic}, we further conducted evaluations on four synthetic functions—Ackley, Levy, Quadratic, and Rosenbrock—utilizing the $\gR$-AdaZO update rule. The results are presented in Figure \ref{fig:synthetic_radazo}. Notably, the performance of Vanilla ZOO and ZoHS is highly similar, which indicates that ZoHS does not confer any additional advantage within the $\gR$-AdaZO setting. Furthermore, the performance of \ours{} w/ and w/o historical information is closely comparable, suggesting that \ours{} w/o history is sufficiently effective for practical application under the $\gR$-AdaZO framework.

\begin{figure}[t]
    \centering
    \includegraphics[width=1.0\textwidth]{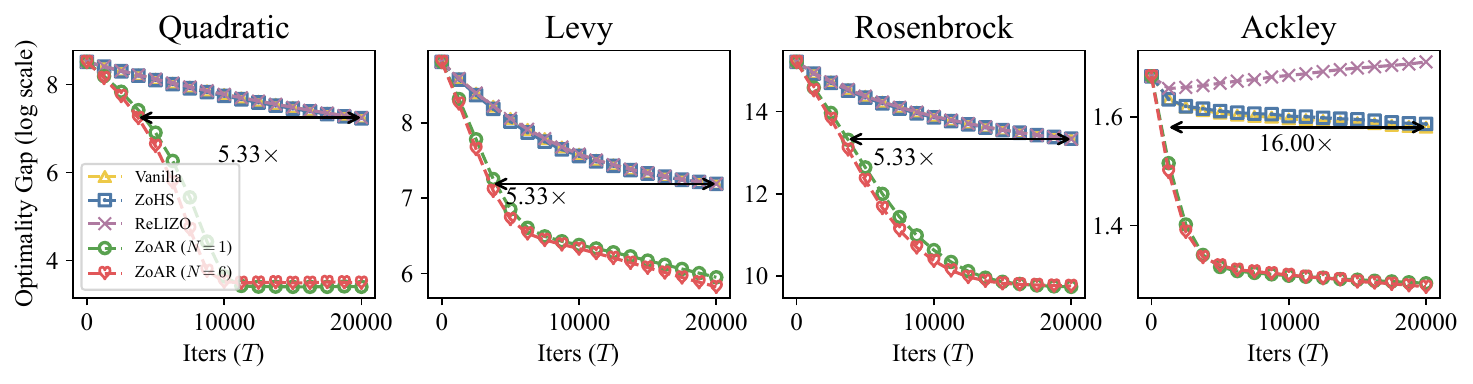}
    \vspace{-7mm}
    \caption{Comparison of convergence among different zeroth-order optimization algorithms on four synthetic functions under $\gR$-AdaZO setting. The $x$-axis denotes the number of iterations, and the $y$-axis denotes the gap between the current function value and the optimal function value, i.e. $F(\vtheta) - \min_{\vtheta'} F(\vtheta')$. All curves are averaged over 5 independent runs.}
    \label{fig:synthetic_radazo}
\end{figure}

\subsection{Memory-Efficient LLM Fine-Tuning}\label{sec:mezo}
The pursuit of memory-efficient fine-tuning for large language models (LLMs) has recently incorporated zeroth-order optimization techniques (\citep{mezo}). However, conventional zeroth-order optimization methods typically exhibit increased variance in gradient estimation, which can adversely affect the convergence of LLM fine-tuning. To mitigate this variance, \ours{} reuses historical information without incurring the additional cost of new queries. In this section, we fine-tunes the OPT-1.3B and OPT-13B models on the SST2 and COPA datasets, respectively, employing the $\mathcal{R}$-AdaZO update rule (refer to \ref{subsec:mezo_details} for more details). \ours{} is compared against the vanilla Zeroth-Order Optimization (ZOO) method, which served as the baseline. The results, presented in Figure \ref{fig:mezo_radazo}, demonstrate that \ours{} outperforms the vanilla ZOO method, particularly for the smaller OPT-1.3B model. Furthermore, the convergence rate of \ours{} incorporating historical information surpasses that of the variant without historical information, suggesting the beneficial role of historical data in LLM fine-tuning.

\begin{figure}[t]
    \centering
    \includegraphics[width=1.0\textwidth]{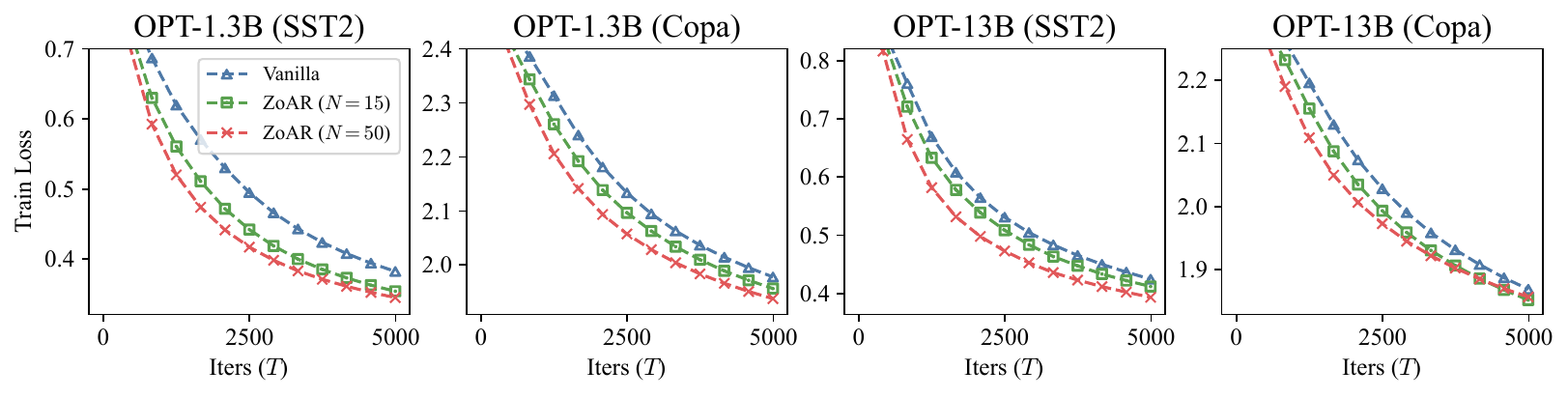}
    \vspace{-7mm}
    \caption{Training loss comparison between Vanilla ZOO and \ours{} for the LLM fine-tuning under different model sizes on SST2 and Copa datasets. Each curve is averaged over 3 independent runs.}
    \label{fig:mezo_radazo}
\end{figure}

\section{Limitations and Broader Impact}\label{appx:limit}

\ours{} is an excellent variance reduction ZOO method, which can not only reduce the memory cost, but also increase the convergence rate. Therefore, \ours{} is suitable for many variance dominate tasks, especially for LLM fine-tuning. Besides, \ours{} works well even with a large smoothing parameter $\mu = 0.01$ or $\mu = 0.1$, which is much larger than the commonly used $\mu = 0.001$ in Vanilla ZOO. This is because \ours{} reuses the queries to smooth the gradient estimation, which is equivalent to using a smaller smoothing parameter $\mu$. This suggests that \ours{} is suitable to some non-smoothness objective function, such as quantized function in quantization aware training (QAT) field. Recent study  \citep{zhou2025quzo} have combined the ZOO with QAT to avoid the inaccuracy occured by straight through estimator (STE). However, quantized function is actually a multiple step function, where small smoothing parameter $\mu$ would not sufficiently change the quantized function value, especially for ultra-low precision (such as FP4), and often leads to worse convergence. \ours{} would be a good choice for ultra-low precision QAT, since it can use a large smoothing parameter $\mu$ to smooth the quantized function, which is left for future work.

Besides, despite its effectiveness, \ours{} presents several limitations. First, similarly with some variance reduction techniques, such as \citep{radazo}, \ours{} reuse historical queries, which introduce additional bias (Thm. \ref{thm:bias-variance-decomp}), potentially leading to inaccurate descent directions. However, the extra bias is proptotional to the length of the historical gradient, and hence we can introduce linear schudule to dynamically adjust the history length, aiming to reduce or even eliminate the bias. This can be left for future work. Moreover, \ours{} retrive the historical samples from random seed storage, which may cost extra computation when history length is large. This can be solved by utilizing parallel computing techniques or employing dynamic scheduling of history length to improve computational efficiency.

%% file: paper.bbl
\begin{thebibliography}{33}
\providecommand{\natexlab}[1]{#1}
\providecommand{\url}[1]{\texttt{#1}}
\expandafter\ifx\csname urlstyle\endcsname\relax
  \providecommand{\doi}[1]{doi: #1}\else
  \providecommand{\doi}{doi: \begingroup \urlstyle{rm}\Url}\fi

\bibitem[Chen et~al.(2019)Chen, Liu, Xu, Li, Lin, Hong, and Cox]{zo-adamm}
Chen, X., Liu, S., Xu, K., Li, X., Lin, X., Hong, M., and Cox, D.
\newblock Zo-adamm: Zeroth-order adaptive momentum method for black-box optimization.
\newblock In \emph{Proc. {NeurIPS}}, 2019.

\bibitem[Cheng et~al.(2021)Cheng, Wu, and Zhu]{prgf}
Cheng, S., Wu, G., and Zhu, J.
\newblock On the convergence of prior-guided zeroth-order optimization algorithms.
\newblock In \emph{Proc. {NeurIPS}}, 2021.

\bibitem[Flaxman et~al.(2005{\natexlab{a}})Flaxman, Kalai, and McMahan]{bsg}
Flaxman, A., Kalai, A.~T., and McMahan, H.~B.
\newblock Online convex optimization in the bandit setting: Gradient descent without a gradient.
\newblock In \emph{Proc. {SODA}}, 2005{\natexlab{a}}.

\bibitem[Flaxman et~al.(2005{\natexlab{b}})Flaxman, Kalai, and McMahan]{flaxman2005online}
Flaxman, A.~D., Kalai, A.~T., and McMahan, H.~B.
\newblock Online convex optimization in the bandit setting: gradient descent without a gradient.
\newblock In \emph{Proc. {SODA}}, 2005{\natexlab{b}}.

\bibitem[Ghadimi \& Lan(2013)Ghadimi and Lan]{GhadimiL13a}
Ghadimi, S. and Lan, G.
\newblock Stochastic first- and zeroth-order methods for nonconvex stochastic programming.
\newblock \emph{{SIAM} J. Optim.}, 23\penalty0 (4):\penalty0 2341--2368, 2013.

\bibitem[Ghadimi et~al.(2016)Ghadimi, Lan, and Zhang]{GhadimiLZ16}
Ghadimi, S., Lan, G., and Zhang, H.
\newblock Mini-batch stochastic approximation methods for nonconvex stochastic composite optimization.
\newblock \emph{Math. Program.}, 155\penalty0 (1-2):\penalty0 267--305, 2016.

\bibitem[Gu et~al.(2021)Gu, Liu, Zhang, Geng, and Huang]{gu2021optimizing}
Gu, B., Liu, G., Zhang, Y., Geng, X., and Huang, H.
\newblock Optimizing large-scale hyperparameters via automated learning algorithm.
\newblock {arXiv:2102.09026}, 2021.

\bibitem[Hu et~al.(2024)Hu, Shu, Yu, Wu, Lin, Dai, Ng, and Low]{zopo}
Hu, W., Shu, Y., Yu, Z., Wu, Z., Lin, X., Dai, Z., Ng, S.-K., and Low, B. K.~H.
\newblock Localized zeroth-order prompt optimization.
\newblock In \emph{Proc. {NeurIPS}}, 2024.

\bibitem[Jiang et~al.(2024)Jiang, Chen, Pan, Xiang, Lin, Wu, Liu, and Song]{adamu}
Jiang, S., Chen, Q., Pan, Y., Xiang, Y., Lin, Y., Wu, X., Liu, C., and Song, X.
\newblock Zo-adamu optimizer: Adapting perturbation by the momentum and uncertainty in zeroth-order optimization.
\newblock In \emph{Proc. {AAAI}}, 2024.

\bibitem[Kingma \& Ba(2015)Kingma and Ba]{adam}
Kingma, D.~P. and Ba, J.
\newblock Adam: {A} method for stochastic optimization.
\newblock In \emph{Proc. {ICLR}}, 2015.

\bibitem[Lecun et~al.(1998)Lecun, Bottou, Bengio, and Haffner]{lecun1998mnist}
Lecun, Y., Bottou, L., Bengio, Y., and Haffner, P.
\newblock Gradient-based learning applied to document recognition.
\newblock \emph{Proceedings of the IEEE}, pp.\  2278--2324, 1998.

\bibitem[Lian et~al.(2016)Lian, Zhang, Hsieh, Huang, and Liu]{coordinate}
Lian, X., Zhang, H., Hsieh, C., Huang, Y., and Liu, J.
\newblock A comprehensive linear speedup analysis for asynchronous stochastic parallel optimization from zeroth-order to first-order.
\newblock In \emph{Proc. {NIPS}}, 2016.

\bibitem[Liu et~al.(2018{\natexlab{a}})Liu, Kailkhura, Chen, Ting, Chang, and Amini]{0001KCTCA18}
Liu, S., Kailkhura, B., Chen, P., Ting, P., Chang, S., and Amini, L.
\newblock Zeroth-order stochastic variance reduction for nonconvex optimization.
\newblock In \emph{Proc. {NeurIPS}}, 2018{\natexlab{a}}.

\bibitem[Liu et~al.(2018{\natexlab{b}})Liu, Li, Chen, Haupt, and Amini]{0001LCHA18}
Liu, S., Li, X., Chen, P., Haupt, J.~D., and Amini, L.
\newblock Zeroth-order stochastic projected gradient descent for nonconvex optimization.
\newblock In \emph{Proc. {GlobalSIP}}, 2018{\natexlab{b}}.

\bibitem[Malladi et~al.(2023)Malladi, Gao, Nichani, Damian, Lee, Chen, and Arora]{mezo}
Malladi, S., Gao, T., Nichani, E., Damian, A., Lee, J.~D., Chen, D., and Arora, S.
\newblock Fine-tuning language models with just forward passes.
\newblock In \emph{Proc. {NeurIPS}}, 2023.

\bibitem[Nazari et~al.(2020)Nazari, Tarzanagh, and Michailidis]{nazari2020adaptive}
Nazari, P., Tarzanagh, D.~A., and Michailidis, G.
\newblock Adaptive first-and zeroth-order methods for weakly convex stochastic optimization problems.
\newblock {arXiv:2005.09261}, 2020.

\bibitem[Nesterov \& Spokoiny(2017)Nesterov and Spokoiny]{Nesterov2017}
Nesterov, Y.~E. and Spokoiny, V.~G.
\newblock Random gradient-free minimization of convex functions.
\newblock \emph{Found. Comput. Math.}, 17\penalty0 (2):\penalty0 527--566, 2017.

\bibitem[Salimans et~al.(2017)Salimans, Ho, Chen, and Sutskever]{SalimansHCS17}
Salimans, T., Ho, J., Chen, X., and Sutskever, I.
\newblock Evolution strategies as a scalable alternative to reinforcement learning.
\newblock {arXiv:1703.03864}, 2017.

\bibitem[Schulman et~al.(2017)Schulman, Wolski, Dhariwal, Radford, and Klimov]{schulman2017proximal}
Schulman, J., Wolski, F., Dhariwal, P., Radford, A., and Klimov, O.
\newblock Proximal policy optimization algorithms.
\newblock {arXiv:1707.06347}, 2017.

\bibitem[Shu et~al.(2023)Shu, Dai, Sng, Verma, Jaillet, and Low]{zord}
Shu, Y., Dai, Z., Sng, W., Verma, A., Jaillet, P., and Low, B. K.~H.
\newblock Zeroth-order optimization with trajectory-informed derivative estimation.
\newblock In \emph{Proc. {ICLR}}, 2023.

\bibitem[Shu et~al.(2024)Shu, Lin, Dai, and Low]{fzoos}
Shu, Y., Lin, X., Dai, Z., and Low, B. K.~H.
\newblock Federated zeroth-order optimization using trajectory-informed surrogate gradients.
\newblock In \emph{Workshop on Differentiable Almost Everything \normalfont (ICML)}, 2024.

\bibitem[Shu et~al.(2025{\natexlab{a}})Shu, Hu, Ng, Low, and Yu]{ferret}
Shu, Y., Hu, W., Ng, S.-K., Low, B. K.~H., and Yu, F.~R.
\newblock Ferret: Federated full-parameter tuning at scale for large language models.
\newblock In \emph{Proc. {ICML}}, 2025{\natexlab{a}}.

\bibitem[Shu et~al.(2025{\natexlab{b}})Shu, Zhang, He, and Dai]{radazo}
Shu, Y., Zhang, Q., He, K., and Dai, Z.
\newblock Refining adaptive zeroth-order optimization at ease.
\newblock In \emph{Proc. {ICML}}, 2025{\natexlab{b}}.

\bibitem[Stein(1981)]{stein1981estimation}
Stein, C.~M.
\newblock Estimation of the mean of a multivariate normal distribution.
\newblock \emph{The annals of Statistics}, pp.\  1135--1151, 1981.

\bibitem[Sutton \& Barto(2018)Sutton and Barto]{sutton2018reinforcement}
Sutton, R.~S. and Barto, A.~G.
\newblock \emph{Reinforcement learning - an introduction, 2nd Edition}.
\newblock {MIT} Press, 2018.

\bibitem[Sutton et~al.(1999)Sutton, McAllester, Singh, and Mansour]{sutton1999policy}
Sutton, R.~S., McAllester, D.~A., Singh, S., and Mansour, Y.
\newblock Policy gradient methods for reinforcement learning with function approximation.
\newblock In \emph{{Proc. NIPS}}, 1999.

\bibitem[Wang et~al.(2019)Wang, Singh, Michael, Hill, Levy, and Bowman]{wang2019glue}
Wang, A., Singh, A., Michael, J., Hill, F., Levy, O., and Bowman, S.~R.
\newblock {GLUE}: A multi-task benchmark and analysis platform for natural language understanding.
\newblock In \emph{Proc. {ICLR}}, 2019.
\newblock In the Proceedings of ICLR.

\bibitem[Wang et~al.(2024)Wang, Qin, Yang, and Yan]{relizo}
Wang, X., Qin, X., Yang, X., and Yan, J.
\newblock Relizo: Sample reusable linear interpolation-based zeroth-order optimization.
\newblock In \emph{Proc. {NeurIPS}}, 2024.

\bibitem[Williams(1992)]{williams1992simple}
Williams, R.~J.
\newblock Simple statistical gradient-following algorithms for connectionist reinforcement learning.
\newblock \emph{Mach. Learn.}, 8:\penalty0 229--256, 1992.

\bibitem[Zhan et~al.(2024)Zhan, Chen, Ding, Li, and Sun]{unlocking}
Zhan, H., Chen, C., Ding, T., Li, Z., and Sun, R.
\newblock Unlocking black-box prompt tuning efficiency via zeroth-order optimization.
\newblock In \emph{Proc. {EMNLP} (Findings)}, 2024.

\bibitem[Zhang et~al.(2022)Zhang, Roller, Goyal, Artetxe, Chen, Chen, Dewan, Diab, Li, Lin, Mihaylov, Ott, Shleifer, Shuster, Simig, Koura, Sridhar, Wang, and Zettlemoyer]{zhang2022opt}
Zhang, S., Roller, S., Goyal, N., Artetxe, M., Chen, M., Chen, S., Dewan, C., Diab, M., Li, X., Lin, X.~V., Mihaylov, T., Ott, M., Shleifer, S., Shuster, K., Simig, D., Koura, P.~S., Sridhar, A., Wang, T., and Zettlemoyer, L.
\newblock Opt: Open pre-trained transformer language models, 2022.

\bibitem[Zhang et~al.(2024)Zhang, Li, Hong, Li, Zhang, Zheng, Chen, Lee, Yin, Hong, Wang, Liu, and Chen]{revisiting}
Zhang, Y., Li, P., Hong, J., Li, J., Zhang, Y., Zheng, W., Chen, P., Lee, J.~D., Yin, W., Hong, M., Wang, Z., Liu, S., and Chen, T.
\newblock Revisiting zeroth-order optimization for memory-efficient {LLM} fine-tuning: {A} benchmark.
\newblock In \emph{Proc. {ICML}}, 2024.

\bibitem[Zhou et~al.(2025)Zhou, Yang, Zhen, Liu, Zhao, Banijamali, Mouchtaris, Wong, and Zhang]{zhou2025quzo}
Zhou, J., Yang, Y., Zhen, K., Liu, Z., Zhao, Y., Banijamali, E., Mouchtaris, A., Wong, N., and Zhang, Z.
\newblock Quzo: Quantized zeroth-order fine-tuning for large language models, 2025.

\end{thebibliography}
